\def\year{2017}\relax
\documentclass[letterpaper]{article}
\usepackage{aaai17}
\usepackage{times}
\usepackage{helvet}
\usepackage{courier}
\nocopyright

\usepackage[T1]{fontenc}
\usepackage{amsmath,graphicx}
 \usepackage{amssymb,stmaryrd,amsmath,amsfonts,rotating,mathrsfs}
 \usepackage{MnSymbol}
\usepackage[utf8]{inputenc}
\usepackage{float}
\usepackage{mathtools}
\usepackage{amsthm}
\usepackage{amssymb,xspace}
\usepackage{epstopdf,epsfig,color,url}
\usepackage{setspace}
\usepackage{tikz}
\usetikzlibrary{trees}
\usepackage{algorithmic}

\usepackage{ifthen}
\usepackage{substr,atbegshi,etoolbox,balance}

\newcommand{\AlgDC}{{\textsc{DC}}\xspace}

\newcommand{\AlgRC}{{\textsc{Repetitive-DC}}\xspace}

\newcommand{\AlgRS}{{\textsc{Random-Sampling}}\xspace}
\newcommand{\AlgUS}{{\textsc{Uncertainty-Sampling}}\xspace}
\newcommand{\AlgQB}{{\textsc{Query-by-Bagging}}\xspace}
\newcommand{\AlgSP}{{\textsc{Spectral}}\xspace}

 \newcommand{\AlgDCt}{{\textsc{$\text{DC}^2$}}\xspace}

\floatstyle{ruled}
\newfloat{algorithm}{tbp}{loa}
\providecommand{\algorithmname}{Algorithm}
\floatname{algorithm}{\protect\algorithmname}

\newcounter{ex}
\theoremstyle{plain}
\newtheorem{thm}{\protect\theoremname}
  \theoremstyle{plain}
  
  \theoremstyle{plain}
  \newtheorem{lem}[thm]{\protect\lemmaname}
    \theoremstyle{definition}
    \newtheorem{example}[ex]{\protect\examplename}
    \theoremstyle{remark}

\newcommand{\R}{\mathbb{R}}
\usepackage[caption=false,font=footnotesize]{subfig}

\makeatother

  \providecommand{\lemmaname}{Lemma}
  \providecommand{\propertyname}{Property}
\providecommand{\theoremname}{Theorem}
\providecommand{\examplename}{Example}

\global\long\def\p{\mathrm{Pr}}

\global\long\def\argmax{\operatorname*{arg\, max}}

\global\long\def\span{\operatorname{span}}
\global\long\def\sgn{{\textbf{sign}}}
\global\long\def\h{h^{*}}

\global\long\def\ep{\rho}
\global\long\def\t{T_{\epsilon,\delta}}

\global\long\def\r{\rho}

\title{Near-Optimal Active Learning of Halfspaces via Query Synthesis in the Noisy Setting}
\author{Lin~Chen$ ^{1,2} $ \and Hamed~Hassani$ ^3 $ \and Amin~Karbasi$ ^{1,2} $\\
$^1$Department of Electrical Engineering,
$^2$Yale Institute for Network Science, Yale University\\
$^3$Computer Science Department, ETH Z{\"u}rich\\
\{lin.chen, amin.karbasi\}@yale.edu, hamed@inf.ethz.ch
	}
\begin{document}
	\maketitle
	\begin{abstract}
In this paper, we consider the problem of actively learning a linear classifier through query synthesis where the learner can construct  artificial queries in order to estimate the true decision boundaries. This problem has recently gained a lot of interest in automated science and adversarial reverse engineering for which only heuristic algorithms are known. In such applications, queries can be constructed \textit{de novo} to elicit information  (e.g., automated science) or to evade detection with minimal cost (e.g., adversarial reverse engineering).   
We develop a general framework, called \emph{dimension coupling} (\AlgDC), that 1)  reduces a $d$-dimensional learning problem to  $d-1$ low-dimensional sub-problems, 2) solves each sub-problem efficiently, 3) appropriately aggregates the results and outputs a linear classifier, and 4) provides a theoretical guarantee for all possible schemes of aggregation.  
The proposed method is proved resilient to noise.
	   We show that  the \AlgDC framework avoids the curse of dimensionality: its computational complexity 
	   scales  linearly with the dimension.  Moreover,
	      we show that the query complexity of \AlgDC is near optimal (within a constant factor of the optimum algorithm). 
	 To further support our theoretical analysis, we compare the performance of \AlgDC with the existing work. We observe that \AlgDC consistently outperforms the prior arts in terms of query complexity while often running orders of magnitude  faster. 
	\end{abstract}
\section{Introduction}
In contrast to  the passive model of supervised learning,  where all the labels  are provided without any interactions with the learning mechanism, the key insight in active learning is that the learning algorithm can perform significantly better if it is allowed to   choose which data points to label. This approach has found far-reaching  applications, including the classical problems in AI (e.g., classification \cite{tong2002support}, information retrieval \cite{tong2001support}, speech recognition \cite{hakkani2002active}) as well as the  modern ones (e.g., interactive recommender systems \cite{karbasi2012comparison} and optimal decision making \cite{javdani2014near}). 
In all the above applications,  the unlabeled data are usually abundant and easy to obtain, but training labels are either  time-consuming or expensive to acquire (as they  require asking an expert).  

Throughout this paper, our objective is to actively learn an unknown
halfspace $H^{*}=\{x\in\mathbb{R}^{d}:\left\langle h^{*},x\right\rangle >0\}$ via \emph{query synthesis} (a.k.a.~membership queries), where  $\left\langle \cdot,\cdot\right\rangle$ denotes the standard inner
product of the Euclidean space and $h^{*}$ is the unit normal vector of the halfspace we want to learn. We would like to note that learning the halfspace $ H^* $ is mathematically equivalent to learning its unit normal vector $ h^* $; therefore we focus on learning $ h^* $ hereinafter. In addition, it should be noted that using the kernel trick we can easily extend the halfspace learning to more complex (e.g., non-linear) decision boundaries \cite{shawe2004kernel}.

The hypothesis space $ \mathcal{H} $, which consists of all possibilities of unit normal vectors, is the unit sphere $ S^{d-1}=\{x\in\mathbb{R}^{d}:\left\Vert x\right\Vert =1\} $, where $\Vert \cdot \Vert$  denotes the standard Euclidean norm.

In active learning of halfspaces via query synthesis, the algorithm is allowed to query whether any point $ x $ in $ \mathbb{R}^d $ resides in the true halfspace. When the algorithm queries $ x $, the true outcome is $ \sgn(\left\langle h^{*},x \right\rangle )\in \{1,-1\}$. When $ \sgn(\left\langle h^{*},x \right\rangle )=1$, it means that $ x\in {H}^* $; otherwise, $ x\notin {H}^* $. We should note here that  the only information we obtain from a query is the \emph{sign} of the inner product rather than the value.
For example, the queries of the form $ \sgn(\left\langle h^{*}, e_i \right\rangle )$, where $e_i$ is the $i$th standard basis vector, will only reveal the 
\emph{sign} of the $i$th component of $h^*$ (and nothing further about its value). 

 In the noiseless setting, we observe the true outcome of the query,  
i.e. $\sgn \left\langle h^{*},x \right\rangle \in\{1,-1\}$. In the noisy setting, 
the outcome is a flipped version of the true sign with independent flip probability $\rho$. 
That is, denoting the outcome by $y$ we have
$y \in \{-1,1\}$ and  $\p[y \neq \sgn \left\langle h^{*},x \right\rangle] \triangleq \rho <1/2$.

Since the length of the selected vector $x$ will not affect the outcome of the query, we only query
the points on the unit sphere $S^{d-1}=\{x\in\mathbb{R}^{d}:\left\Vert x\right\Vert =1\}.$
Hence, we term $\mathcal{X}=S^{d-1}$ as the \emph{query space}.

 Given $\epsilon, \delta >0$, we would like to seek an active learning algorithm 
 that (i) adaptively selects vectors $x_1, x_2, \ldots\in \mathcal{X}$, (ii) observes the (noisy) 
 responses to each query $\sgn \langle h^*, x_i \rangle$, (iii) and outputs, 
 using as few queries as possible, an estimate $\hat{h}$ of $h^*$ such that $\Vert \hat{h} - h^* \Vert < \epsilon$ with probability at 
 least $1-\delta$. 

Our main contribution in this paper is to develop a noise resilient active learning algorithm that has access to \emph{noisy} membership queries. To the best of our knowledge, we are the first to show a near-optimal algorithm that outperforms in theory and practice the naive repetition mechanism and the recent spectral heuristic methods \cite{AAAI}. Specifically, we develop a framework, called Dimension Coupling (\AlgDC), with the following guarantees. Its query complexity is ${O}(d(\log \frac{1}{{\epsilon}} + \log \frac{1}{{\delta}}))$ and its computational complexity  is ${O}(d(\log \frac{1}{{\epsilon}} + \log \frac{1}{{\delta}})^2)$. In particular, in the noiseless setting ($ \rho=0 $), both  its  computational complexity  and  query complexity are  ${O}(d \log\frac{1}{{\epsilon}})$.  Note that in both settings the computational complexity scales linearly with the dimension. Moreover, the query complexity in both settings is near-optimal. 
Our empirical experiments demonstrate that \AlgDC runs orders of magnitude faster than the existing methods.

The rest of the paper is structured as follows. 
In Section~\ref{sec:DC2}, we start with investigating this problem in the $ 2 $-dimensional case and present an algorithm called \AlgDCt. Then in Section~\ref{sec:DC} we generalize it to the  $ d $-dimensional case and present a general framework called \AlgDC.
Empirical results are shown in Section~\ref{sec:expr}. We extensively review related literature in Section~\ref{sec:related}.


\section{\AlgDCt: Solving the $ 2 $-Dimensional Problem} \label{sec:DC2}
To gain more intuition before studying the general $ d $-dimensional problem, it might be beneficial to study a special case where the dimension is two. In other words, we study in this section how to learn the normalized projection of the true unit normal vector $ h^*\in \mathbb{R}^d $ onto $ \span\{e_1,e_2\} $, where $ e_1 , e_2\in \mathbb{R}^d $ are two orthonormal vectors and $ \span\{e_1,e_2\} $ is the linear subspace spanned by $ e_1 $ and $ e_2 $. We should note here that the underlying space is still $ d $-dimensional (i.e., $ \mathbb{R}^d $) but our goal is not to learn $ h^* $ \textit{per se} but its normalized projection onto a $ 2 $-dimensional subspace.

Formally, given two orthonormal vectors $e_1, e_2$ we denote the (normalized) projection of $h^*$ onto $\span\{e_1, e_2\}$ by $h^\perp$, i.e., 
\begin{equation}
h^\perp = \frac{\left\langle h^{*},e_{1}\right\rangle e_{1}+\left\langle h^{*},e_{2}\right\rangle e_{2}}{\left\Vert \left\langle h^{*},e_{1}\right\rangle e_{1}+\left\langle h^{*},e_{2}\right\rangle e_{2}\right\Vert_2 } . 
\end{equation}
Our objective 
 is to find a unit vector $\hat{e} \in \span\{e_1, e_2\}$ such that $\Vert \hat{e} - h^\perp \Vert < \epsilon$.
In fact, we require the latter to hold with probability at least $1-\delta$.  

We should emphasize that noise, characterized by independent flip probability $ \rho $, is generally present. In the $ 2 $-dimensional problem, one may propose to use the simple binary search (a detailed discussion with examples is presented in Appendix~\ref{sub:noiseless}) to find a unit vector $ \hat{e} $ that resides $ \epsilon $-close to $ h^\perp $. To make it noise-tolerant, when the binary search algorithm queries a point, say $ x_i $, we query it $ R $ times to obtain $ R $ noisy versions of $ \sgn\left\langle h^*,x_i\right\rangle $ and view the majority vote of the noisy versions as the true outcome~\cite{kaariainen2006active,karp2007noisy,nowak2011geometry}. We call this method \emph{repetitive querying}. However, its query complexity is $ O(\log(1/\epsilon)(\log \log (1/\epsilon)+\log (1/\delta )) $, which is suboptimal both theoretically (we will prove this bound in \textbf{Appendix~\ref{sub:rep}}) and empirically (referred to as \AlgRC in Section~\ref{sec:expr}).

As a result, instead, we will present a Bayesian algorithm termed \AlgDCt that solves this $ 2 $-dimensional problem with query complexity $ O(\log (1/\epsilon)+\log(1/\delta) ) $.
Recall that any unit vector inside $\span\{e_1, e_2\}$, e.g., $h^\perp$, can equivalently be represented as a pair $(c_1, c_2)$ on the two-dimensional unit circle $S^1$ (e.g., $h^\perp = c_1 e_1 + c_2 e_2$ and $c_1^2 + c_2^2 = 1$). To simplify notation, we use a point $(c_1,c_2) \in S^1$ and its corresponding unit vector $c_1 e_1 + c_2 e_2$ interchangeably.  In this setting, it is easy to see that for any $x \in \span\{e_1, e_2 \}$  
\vspace{-.1cm}
\begin{equation} \label{eq-proj}
\sgn \left\langle x, h^{*}\right\rangle =  \sgn  \left\langle x,  h^\perp   \right\rangle.
\end{equation}
We take a Bayesian approach. In the beginning, when no queries have been performed, \AlgDCt assumes no prior information about the vector $h^\perp$. Therefore, it takes the uniform distribution on $S^1$ (with pdf $p_0(h) = \frac{1}{2 \pi}$) as its prior belief about $h^\perp$. After performing each query, the posterior (belief) about $h^\perp$ will be updated according to the observation. 
We let $p_m(h)$ denote the (pdf of the) posterior after performing the first $m$ queries.  
In this manner, \AlgDCt runs in total of $T_{\epsilon, \delta}$ rounds, where in each round a specific query is selected and  posed to the oracle. The number $T_{\epsilon, \delta}$ will be specified later (see Theorem~\ref{thm:noisy}).
 Upon the completion of round $T_{\epsilon, \delta}$, the algorithm returns as its final output a vector $\hat{e} \in S^1$ that maximises the posterior pdf $p_{T_{\epsilon, \delta}}(h)$. If there are multiple such maximisers, it picks one arbitrarily.  We now proceed with a detailed description of \AlgDCt (a formal description is provided in Algorithm~\ref{alg:DC2}). 
\begin{algorithm}[tb]
	
	\begin{algorithmic}[1]
		\REQUIRE orthonormal vectors
		$e_{1},e_{2}$, estimation error  at most $\epsilon$, success probability at least $1-\delta$.
		\ENSURE a unit vector $ \hat{e} $ which is an estimate for the normalized orthogonal projection of $ h^* $ onto $ \span\{e_1,e_2\} $.
		\STATE Set $p_0(h)$ to be uniform, i.e., $\forall h \in S^1: p_0(h) = 1/2\pi$. 
		\FOR{ $m=1$ \TO $T_{\epsilon, \delta}$ }
		
			\STATE Find a vector $x_m \in S^1$ which is a solution to the following equation: $\int_{S^1}  \sgn \left \langle x, h \right \rangle p_{m-1}(h) dh = 0.$ If there are multiple solutions, choose one arbitrarily. 
			\STATE Ask from the oracle the value of $\sgn \left \langle x_m, h^* \right \rangle $.
			\STATE Based on the response obtained from the oracle, update the distribution $p_{m-1}(h)$ to $p_{m}(h)$.
		
		\ENDFOR 
		\RETURN $\hat{e} = \argmax_{h \in S^1} p_{T_{\epsilon, \delta}} (h)$. 
	\end{algorithmic}
	
%
%

\protect\caption{\AlgDCt \label{alg:DC2}}

\end{algorithm}

As shown in Algorithm~\ref{alg:DC2}, at each round, say round $ m+1 $, the algorithm maintains and updates the distribution $ p_{m} $ that encodes its current belief in the true location of $ h^\perp $. 
 We should note here that these distributions can 
  be stored efficiently  and as a result the vector $x_{m+1}$ can be computed efficiently. Indeed, (the pdf of) $p_m$ is 
piecewise constant on the unit circle (see Figure~\ref{ex-noisy}).
 More precisely, at any round $m$, there are at most $2m$ points $u_1, u_2, \cdots, u_{2m}$ that are ordered clock-wise on the unit-circle and 
$p_m$ is constant when restricted to each of the sectors $[u_i, u_{i+1})$.   The piecewise constant property of the pdf of $ p_m $ can be established by induction on $ m $. Recall that the initial distribution $ p_0 $ is uniform and thus piecewise constant. The Bayesian update step (line 5 of Algorithm~\ref{alg:DC2}) preserves this property when the algorithm updates the distribution $ p_m(h) $ to $ p_{m+1}(h) $. We will show why this is true when we discuss the Bayesian update step in detail.

At round $m+1$, in order to find $x_{m+1}$ (see line 3 of Algorithm~\ref{alg:DC2}), \AlgDCt first finds a line that passes through the centre of $S^1$ and cuts $S^1$ into two ``halves'' which have the same measure with respect to $p_m$.  Note that finding such a line can be done in $O(m)$ steps because $p_m$ has the piecewise constant property.   Once such a line is found, it is then easy to see that $x_{m+1}$ can be any of the two points orthogonal to the line.  As a result, \AlgDCt at round $m+1$ can find $x_{m+1}$ in $O(m)$ operations. We denote the half-circle containing $x_{m+1}$ by $R^+$ and the other half by $R^-$.  We refer to Figure~\ref{ex-noisy} for a schematic illustration. 

The key step in Algorithm~\ref{alg:DC2} is the Bayesian update (line 5).
Once a noisy response to the query $\sgn \left\langle x_{m+1}, h^{*}\right\rangle$ is obtained (line 4)), the probability distribution $p_{m}$ will be updated to $p_{m+1}$ in the following way. First, consider the event that the outcome of $\sgn \left\langle x_{m+1}, h^{*}\right\rangle$ is $+1$. We have
$p_m (\sgn \left\langle x_{m+1}, h^{*}\right\rangle =+1 ) = (1-\rho) \,\,p_m (R^+) + \rho \, \, p_m(R^-) = 1/2$,
and similarly $p_m (\sgn \left\langle x_{m+1}, h^{*}\right\rangle =-1 )  = 1 / 2$. Therefore, by Bayes theorem we obtain the following update rules for $p_{m+1}$. If we observe that $\sgn \left\langle x_{m+1}, h^{*}\right\rangle =+1$, then for $h \in R^+$ we have $$p_{m+1}(h) = 2(1-\rho)p_{m}(h) $$ and for $h \in R^-$ we have $$p_{m+1}(h) = (2\rho) p_{m}(h) .$$
Also, if we observe that $\sgn \left\langle x_{m+1}, h^{*}\right\rangle =-1$, then for $h \in R^+$ we have $$ p_{m+1}(h) = (2 \rho) p_{m}(h) $$ and for $h \in R^-$ we have $$p_{m+1}(h) = 2(1-\rho)p_{m}(h) . $$
Note that the factor of 2 here is due to the normalization.
It is easy to verify  that $p_{m+1}$ is also a piecewise constant distribution (now on $2(m+1)$ sectors; see Figure~\ref{ex-noisy}). 

Theorem~\ref{thm:noisy} shows that after $ T_{\epsilon,\delta}= O( \log \frac{1}{\epsilon} +\log \frac{1}{\delta})$ rounds, with probability at least $ 1-\delta $, \AlgDCt outputs a unit vector $\hat{e} \in \span\{e_1, e_2\}$ such that $\Vert \hat{e} - h^\perp \Vert < \epsilon$. Also, as discussed above, the computational complexity of \AlgDCt is $O(\t^2)$, i.e., ${O}( (\log \frac{1}{\epsilon} + \log \frac{1}{\delta} )^2)$.
\begin{figure}[t]
	\begin{center} 
		\includegraphics[width=7cm]{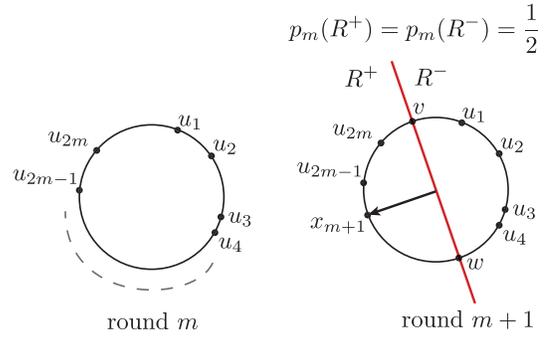}
		\vspace{-.3cm}
	\end{center}
	\caption{\footnotesize Upon the completion of round $m$ (left figure), the distribution (pdf of) $p_m$ is constant over each of the sectors $[u_i, u_{i+1})$. In the next round (right figure), 
		in order to find $x_{m+1}$, \AlgDCt first finds a diagonal
		line (red line) which separates two half-circles ($R^+$ and $R^-$) that each has measure $1/2$ w.r.t $p_m$. The vector $x_{m+1}$ will then be one of the two points on the unit circle that are orthogonal to this line. For updating $p_m$ to $p_{m+1}$, we note that  all the points inside $R^+$  get the same factor (either $2\rho$ or $2(1-\rho)$ depending on the outcome of the query). The same is true for $R^-$. Thus, $p_{m+1}$ is again a piecewise constant pdf but now on $2(m+1)$ sectors.}\label{ex-noisy}
	\vspace{-.2cm}
\end{figure}
\begin{thm} \label{thm:noisy}
	\emph{\textbf{(Proof in Appendix~\ref{sub:noisy})}} When the independent flip probability is $\rho$, having 
	\begin{equation} \label{T}
	T_{\epsilon, \delta} \geq M + \max\{T_0 , T_1 , T_2, T_3\} = O( \log \frac{1}{\epsilon} +\log \frac{1}{\delta})
	\end{equation} 
	is sufficient to guarantee that \AlgDCt outputs with probability at least $1-\delta$  a vector that is within a distance $\epsilon$ of $h^\perp$. 
	Here, we have $M = \lceil \frac{2\log \frac{2}{\delta} } { -\log(4\rho(1-\rho)) } \rceil $, $T_0 = \frac{8 \log \frac{2}{\delta}}{\log(2(1-\rho)) }$, $T_1 =\frac{8 \log \frac{1}{8 \pi \epsilon}}{\log(2(1-\rho)) }$, $T_2 = \frac{8}{\log(2(1-\rho))} \left(\log (2M) + \log (\frac{4}{\log(2(1-\rho))} ) \right) $ and  $T_3 =  \frac{ 24\rho \log^2 \frac{1-\rho}{\rho}}{\log^2(2(1-\rho))} \left(\log(M) + \log(\frac{4}{\delta})\right)$. 
\end{thm}

We would like to remark that when the independent flip probability $ \rho $ is $ 0 $ (i.e., in the noiseless case), the algorithm \AlgDCt reduces to the binary search.
If we let $T_{\epsilon, \delta} = \lceil\log_{2}\frac{\pi}{\epsilon}\rceil$, then \AlgDCt outputs a vector that is within a distance $\epsilon$ of $h^\perp$. We present a detailed discussion with examples in \textbf{Appendix~\ref{sub:noiseless}}.

A few comments are in order: The above guarantee for \AlgDCt holds with probability one and thus the parameter $\delta$ is irrelevant in the noiseless setting. Furthermore, during each round of \AlgDCt, the distribution $p_m$  can be represented by only two numbers (the starting and ending points of the sector $R_m$), and the vector $x_m$ can be computed  efficiently (it is the orthogonal vector to the midpoint of $R_m$). Therefore,  assuming one unit of complexity for performing the queries, \AlgDCt can be implemented with complexity $O(T_{\epsilon, \delta})$, i.e., $ O(\log(1/\epsilon)) $.

\section{Dimension Coupling Based Framework} \label{sec:DC}

In Section~\ref{sec:DC2}, we devise an algorithm, called  \AlgDCt$(e_1, e_2, \epsilon, \delta)$, that takes as input two orthonormal vectors $e_1, e_2$, uses noisy responses to queries of the form $ \sgn \left\langle x, h^* \right\rangle$, and outputs with probability at least $1-\delta$ a vector $\hat{e}$ with the following three properties:
\begin{eqnarray*}
	& \hat{e} \in \span\{e_1, e_2\},
	\Vert \hat{e} \Vert =1, 
	\Vert \hat{e}  -  \frac{\left\langle h^{*},e_{1}\right\rangle e_{1}+\left\langle h^{*},e_{2}\right\rangle e_{2}}{\left\Vert \left\langle h^{*},e_{1}\right\rangle e_{1}+\left\langle h^{*},e_{2}\right\rangle e_{2}\right\Vert }   \Vert < \epsilon.  
\end{eqnarray*}
In other words, the unit vector  $\hat{e}$ is within a distance $\epsilon$ to the (normalized) projection of $h^{*}$ onto the subspace $\span\{e_1, e_2\}$. 
In the current section, we explain a framework \AlgDC that estimates $h^*$ using  at most $d-1$ calls to \AlgDCt (a formal description will be given in Algorithm~\ref{alg:DC} later).

Let us begin our discussion with a motivating example.
Let $ \{e_1,e_2,\ldots,e_d\} $ be an orthonormal basis of $ \mathbb{R}^d $.
Suppose that $h^*$ has the form 
$h^* = \sum_{i=1}^d c_i e_i$, 
where $\{e_i\}_{i=1}^d$ is an arbitrarily chosen orthonormal basis for 
$\mathbb{R}^d$. We assume w.l.o.g.\ that $h^*$ is normalized (i.e., $\sum_{i=1}^d c_i^2 = 1$). Our objective is then to learn the coefficients $\{c_i\}_{i=1}^d$ within a given precision by using the noisy responses to the selected sign queries. The key insight here is that this task can be partitioned in a divide-and-conquer fashion into many smaller tasks, each involving a few dimensions.  The final answer (the values of $\{c_i\}_{i=1}^d$) will then be obtained by aggregating the answers of these subproblems.  

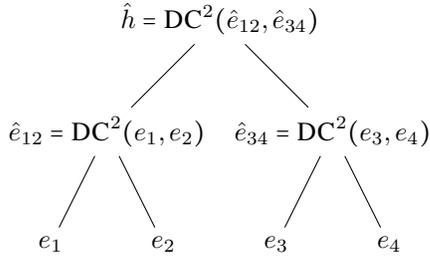
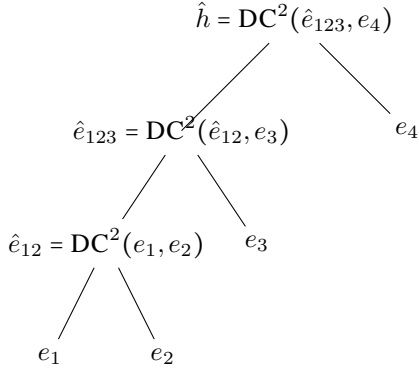
\begin{figure}[t]
	\begin{center}
		\subfloat[Scheme 1: a balanced full binary tree\label{fig:btree1}]{
			\begin{tikzpicture}[level distance=1.5cm,
			level 1/.style={sibling distance=3cm},
			level 2/.style={sibling distance=1.5cm}]
			\node {$\hat{h}=\AlgDCt(\hat{e}_{12},\hat{e}_{34})$}
			child {node {$\hat{e}_{12}=\AlgDCt(e_1,e_2)$}
				child {node {$e_1$}}
				child {node {$e_2$}}
			}
			child {node {$\hat{e}_{34}=\AlgDCt(e_3,e_4)$}
				child {node {$e_3$}}
				child {node {$e_4$}}
			};
			\end{tikzpicture}
			}
	\vspace{-3mm}
	
	\subfloat[Scheme 2: an unbalanced full binary tree\label{fig:btree2}]{
		\begin{tikzpicture}[level distance=1.5cm,
		level 1/.style={sibling distance=3cm},
		level 2/.style={sibling distance=2cm},
		level 3/.style={sibling distance=1.5cm}]
		\node {$\hat{h}=\AlgDCt(\hat{e}_{123},e_4)$}
		child {
			{node {$\hat{e}_{123}=\AlgDCt(\hat{e}_{12},e_3)$}
			child 
				{node {$\hat{e}_{12}=\AlgDCt(e_1,e_2)$}
					child {node {$e_1$}}
					child {node {$e_2$}}
				}
			child {node {$e_3$}}
			}}
		child {node {$e_4$}}
	;
		\end{tikzpicture}
	}	
		
	\end{center}
	\caption{Two possible divide-and-conquer schemes for a $ 4 $-dimensional problem. Each scheme can be represented by a full binary tree of $ 4 $ leaf nodes.}
\end{figure}

\begin{example}\label{ex:btree1}
	Assume $h^* = c_1 e_1 + c_2 e_2 + c_3 e_3 + c_4 e_4$, where $e_i$'s are the standard basis vectors for $\mathbb{R}^4$.
	Define
	\begin{equation*}
		{e}_{12} = \frac{c_1 e_1  +c_2 e_2}{\sqrt{c_1^2 + c_2^2}} \,, \, {e}_{34} = \frac{c_3 e_3 + c_4 e_4}{\sqrt{c_3^2 + c_4^2}}
	\end{equation*}
	Note here that  ${e}_{12}$ is the (normalized) orthogonal projection of $h^*$ onto $\span \{e_1, e_2\}$ and ${e}_{34}$ is the (normalized) orthogonal projection of $h^*$ onto $\span\{e_3, e_4\}$. 
	Consider the following procedure to learn $h^*$: first find out what $e_{12}$ and $e_{34}$ are, and then  use the relation
	$h^* = \sqrt{c_1^2 + c_2^2}  e_{12} + \sqrt{c_3^2 + c_4^2}  e_{34}$
	to find $h^*$ based on the orthonormal vectors $e_{12}, e_{34}$. By this procedure, the original ``four-dimensional'' problem has been broken into three ``two-dimensional'' problems. 
	
	This procedure is illustrated in Figure~\ref{fig:btree1}. We first call $ \AlgDCt(e_1,e_2) $ to obtain an estimate $ \hat{e}_{12} $ for $ e_{12} $; then we call $ \AlgDCt(e_3,e_4) $ to obtain an estimate $ \hat{e}_{34} $ for $ e_{34} $; finally we call $ \AlgDCt(\hat{e}_{12},\hat{e}_{34}) $ to obtain an estimate $ \hat{h} $ for $ h^* $.

\end{example}

\begin{example}\label{ex:btree2}
	For another example of the $ 4 $-dimensional problem discussed in Example~\ref{ex:btree1}, let us consider another scheme illustrated in Figure~\ref{fig:btree2}: We call $ \AlgDCt(e_1,e_2) $ and obtain $ \hat{e_{12}} $ as an estimate for $ e_{12} $; then call $ \AlgDCt(\hat{e}_{12},e_3) $ and obtain $ \hat{e}_{123} $ that estimates the normalized orthogonal projection $ h^* $ onto $ \span\{e_1,e_2,e_3\} $; finally call $ \AlgDCt(\hat{e}_{123},e_4) $ and obtain an estimate for $h^*  $ which we denote by $ \hat{h} $.
\end{example}

Examples~\ref{ex:btree1} and \ref{ex:btree2} show two possibilities of divide-and-conquer schemes for a $ 4 $-dimensional problem. In fact, each scheme corresponds to a full binary tree of $ 4 $ leaf nodes.

For general $d$, the idea is similar: We break the problem into at most $d-1$ ``two-dimensional'' problems that each can be solved efficiently. Again, each divide-and-conquer scheme corresponds to a full binary tree of $ d $ leaf nodes.

Consider the decomposition
$h^* = \sum_{i=1}^d c_i e_i$.  
Without loss of generality, suppose that the first two leaf nodes to be combined are $ e_1 $ and $ e_2 $.
 We can write  
\begin{eqnarray} \label{h-decompose-1}
	h^*  = \sum_{i = 1}^d c_i e_i =  \hat{c}_{12}   \frac{ c_{1} e_{1} +c_{2} e_{2} }{ \sqrt{c_{1}^2 +c_{2}^2 }} + \sum_{i=3}^{d} c_i e_i, 
\end{eqnarray}
where in the last step we have taken $\hat{c}_{12} \triangleq \sqrt{c_{1}^2 +c_{2}^2 }$. Now, note that $ \frac{ c_{1} e_{1} +c_{2} e_{2} }{ \sqrt{c_{1}^2 +c_{2}^2 }}$ is  the normalized orthogonal projection of $h^*$ onto $\span\{e_{1}, e_{2}\}$. Hence, by using \AlgDCt$(e_{1}, e_{2}, \epsilon, \delta)$ we can obtain, with probability at least $1-\delta$, a good approximation  $\hat{e}_{12}$ (within a distance $\epsilon$)  of this projection.
Therefore, for small enough $\epsilon$ we have
$h^* \approx  \hat{c}_{12} \hat{e}_{12}+\sum_{i=3}^{d} c_i e_i.$
Since $h^*$ is now expressed (approximately) in terms of  $d-1$  orthonormal vectors $\{\hat{e}_{12},e_3,e_4,\ldots,e_d\}$, we have effectively reduced the dimensionality of problem from $d$ to $d-1$. The idea is then to repeat the same procedure as in  \eqref{h-decompose-1} to the newly obtained representation of $h^*$. Hence, by repeating this procedure  $d-1$ times in total we will reach a  vector which is  the final approximation of $h^*$. 

We present this general method in Algorithm~\ref{alg:DC}.


%

\begin{algorithm}[tbh]
	\begin{algorithmic}[1]
		\REQUIRE an orthonormal basis $E=\{e_{1},e_{2},\ldots,e_{d}\}$ of $\mathbb{R}^{d}$. 
		\ENSURE a unit vector $ \hat{h} $ which is an estimate for $ h^* $.
		\FOR{$j\gets 1$ \TO $d-1$}
		\STATE Replace any two vectors $ e' $ and $ e'' $ in $E$ with the vector   \AlgDCt$(e',e'', \epsilon, \delta)$. 
		\ENDFOR
		\STATE Let $ \hat{h} $ be the only remaining vector in $ E $.
		\RETURN $ \hat{h} $
	\end{algorithmic}

	
	%
	%
	%
	%
	%
	%
	%
	%
	%
	%
	%
	%
	
	\protect\caption{Dimension Coupling (DC)\label{alg:DC}}
\end{algorithm}

\begin{thm} \label{thm:DC}
	\emph{\textbf{(Proof in Appendix~\ref{sub:DC})}} For  \AlgDC   (outlined in Algorithm~\ref{alg:DC}) and any of its divide-and-conquer scheme represented by a full binary tree, we have: 
	\begin{enumerate}
		\item \AlgDC will call the two-dimensional
		subroutine \AlgDCt  $d-1$ times.
		\item Provided that the output of \AlgDCt is with probability $1-\delta$ within distance $\epsilon$ of the true value and $\epsilon \leq 5/18$, \AlgDC
		ensures an estimation error of at most $5 \epsilon (d-1)$ with probability at least $1-\delta (d-1)$. 
	\end{enumerate}
\end{thm}

As a result of Theorem~\ref{thm:DC}, if we desire the framework \AlgDC to estimate $h^*$ within distance $\tilde{\epsilon}$ and with probability at least $1-\tilde{\delta}$, then it is enough to fix the corresponding parameters of \AlgDCt to $\epsilon = \frac{\tilde{\epsilon}}{5(d-1)}$ and $\delta = \frac{\tilde{\delta}}{d-1}$.

Theorem~\ref{thm:DC} indicates that \AlgDC requires
${O}(d (\log \frac{1}{\epsilon} + \log \frac{1}{\delta} ))$ queries, since each call to \AlgDCt needs $ O(\log \frac{1}{\epsilon} + \log \frac{1}{\delta} ) $ queries. Recall that the computational complexity of \AlgDCt is  ${O}( (\log \frac{1}{\epsilon} + \log \frac{1}{\delta} )^2)$. Hence, \AlgDC has computational complexity 
${O}(d (\log \frac{1}{\epsilon} + \log \frac{1}{\delta} )^2)$. As a special case, if in absence of noise, both the query complexity and time complexity of \AlgDC are $ O(d \log \frac{1}{\epsilon}) $.

\section{Empirical Results}
\label{sec:expr}
\begin{figure*}[ht!]
\vspace{-5mm}
	\subfloat[Noiseless ($d=25$)\label{fig:nf-25d}]{\includegraphics[width=0.33\textwidth]{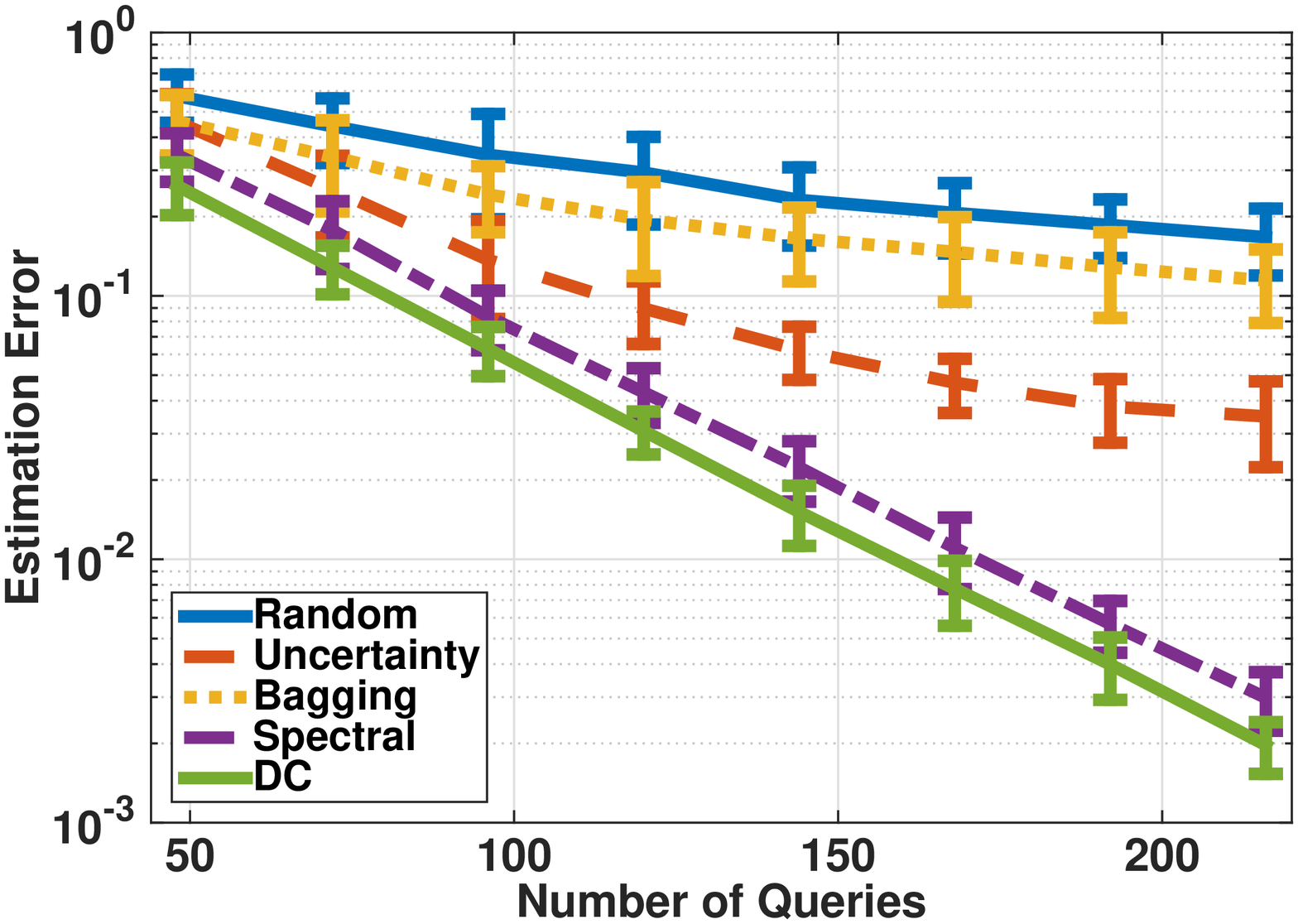}}\subfloat[Noiseless ($d=50$)\label{fig:nf-50d}]{\includegraphics[width=0.33\textwidth]{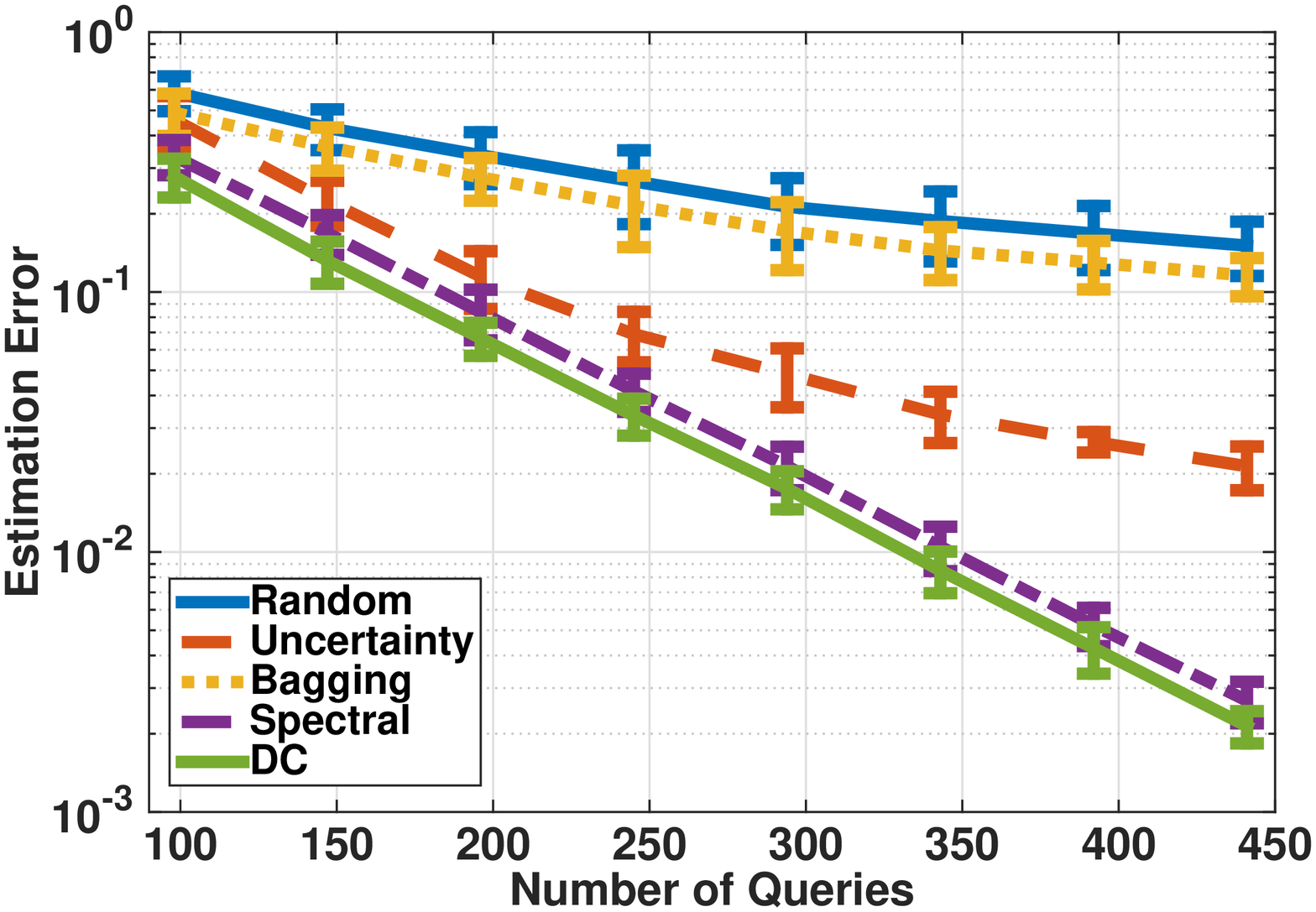}
		
	}
%
\subfloat[Execution time (noiseless)\label{fig:nf-time}]{\includegraphics[width=0.33\textwidth]{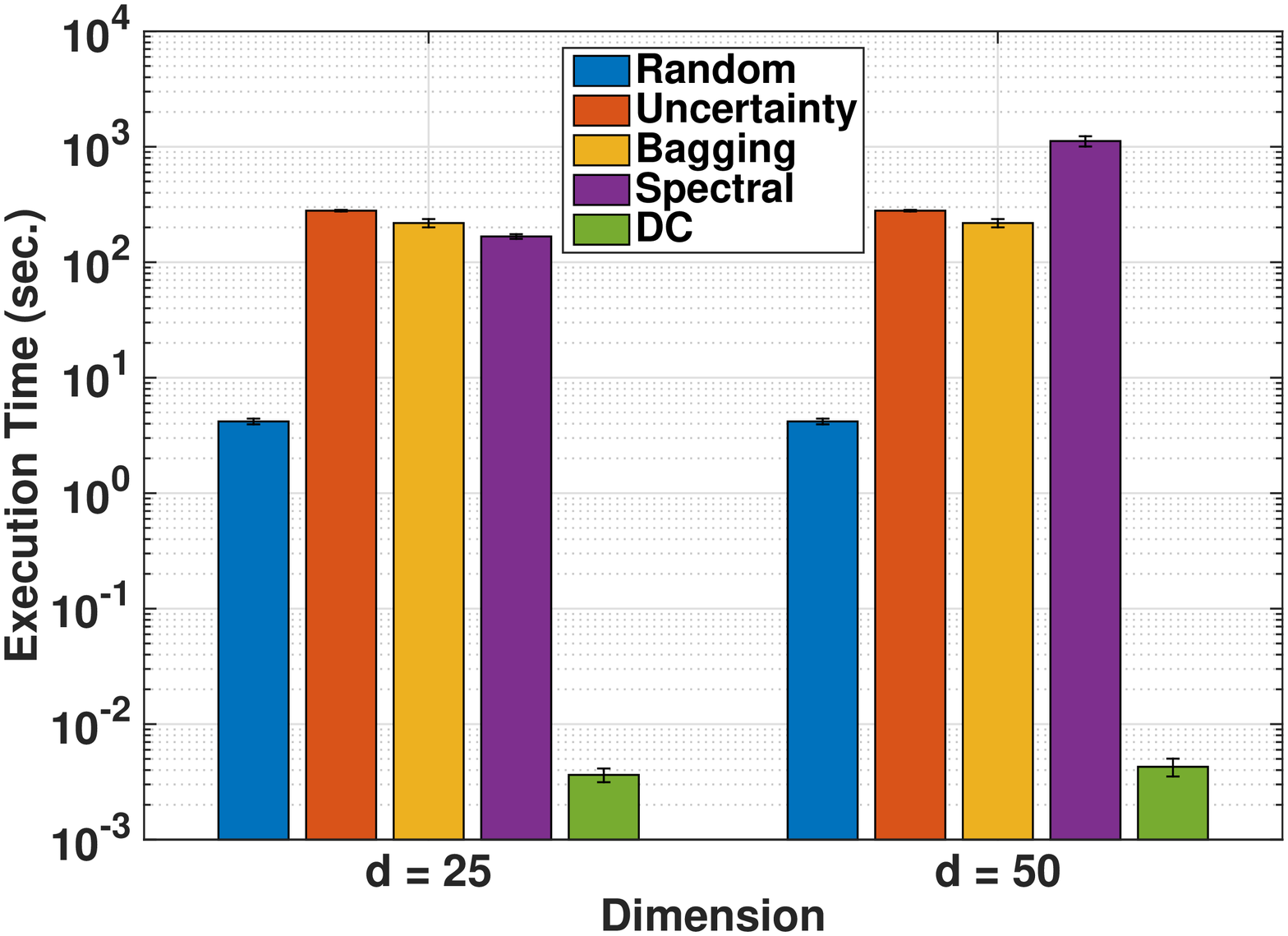}

}
\vspace{-4.5mm}

\subfloat[Noisy ($d=25$)\label{fig:noisy-25}]{\includegraphics[width=0.33\textwidth]{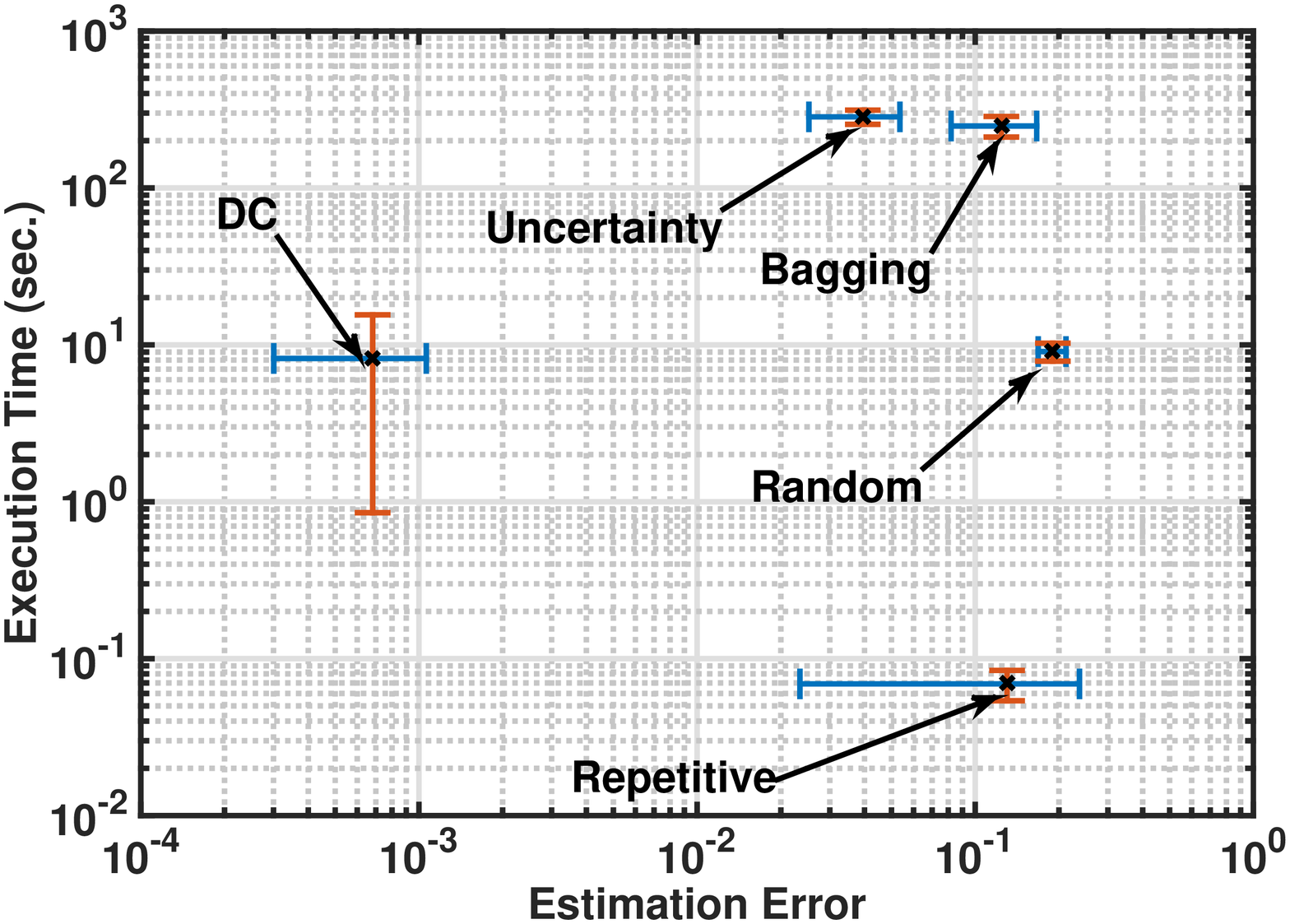}

}\subfloat[Noisy ($d=50$)\label{fig:noisy-50}]{\includegraphics[width=0.33\textwidth]{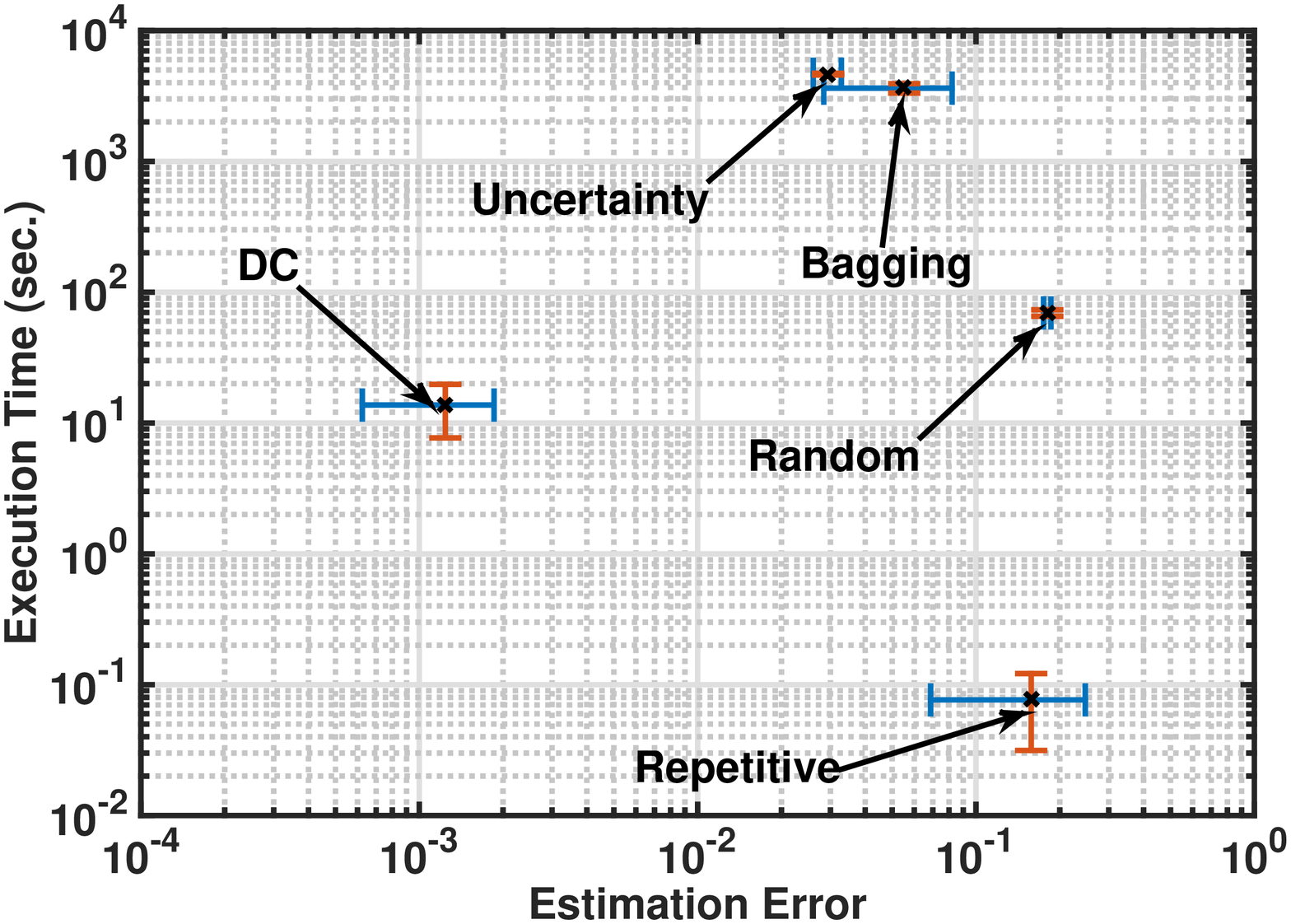}

}
%
\subfloat[Noisy ($d=1000$)\label{fig:noisy-high}]{\includegraphics[width=0.33\textwidth]{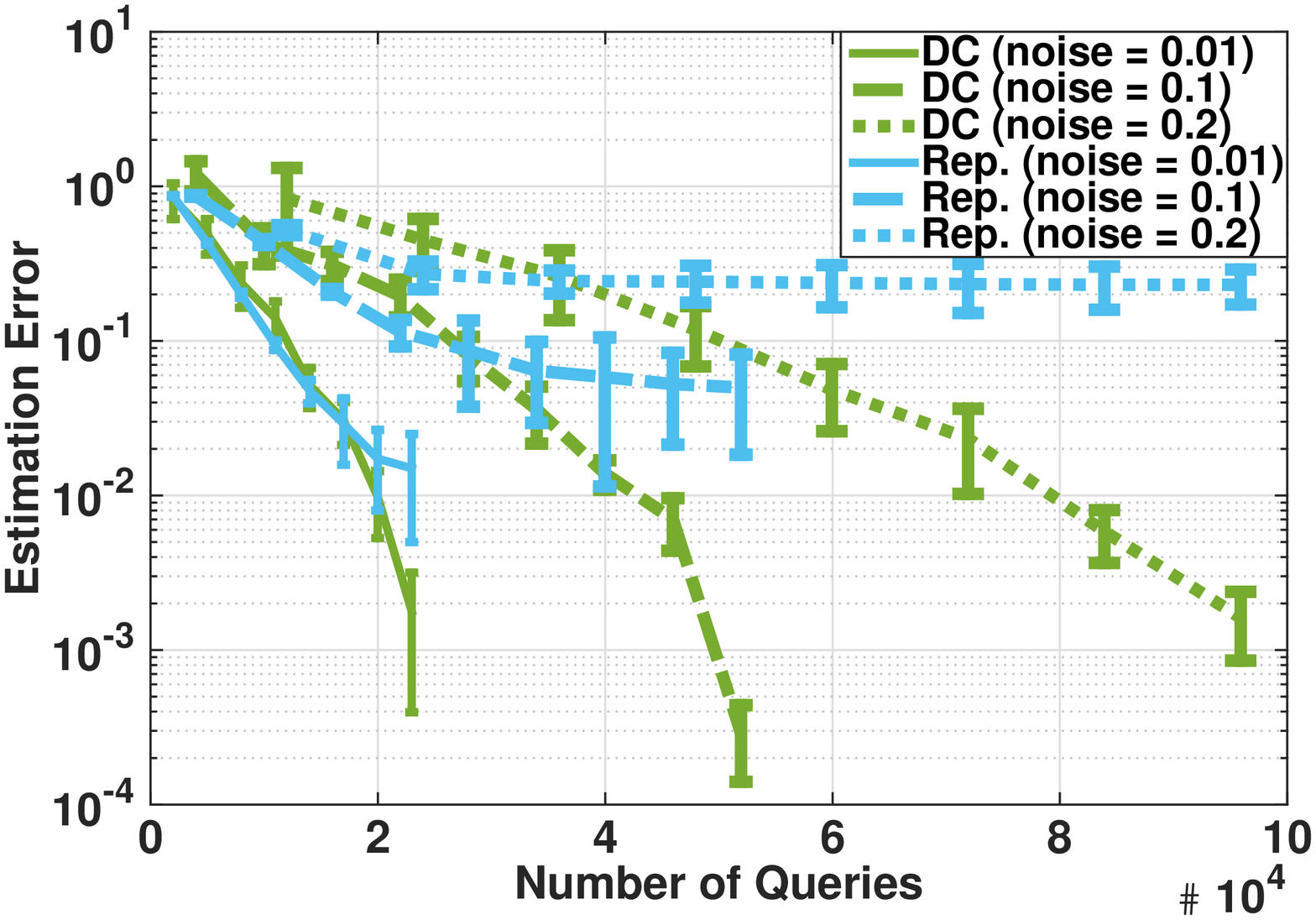}

}
%
\vspace{-2mm}
\protect\caption{Figures~\ref{fig:nf-25d} and \ref{fig:nf-50d}
show the estimation error in the noiseless setting as we increase the number
of  queries, for $d= 25$ and $100$, respectively. Figure~\ref{fig:nf-time} shows the corresponding execution times.  Figure~\ref{fig:noisy-25} and \ref{fig:noisy-50} show the scatter plots of the execution time and the estimation error of different methods
 for $d=25,50$ and the noise level
$\rho=0.1$. We allow each algorithm to use a budget of $ 800 $ and $ 1800 $ queries in Figure~\ref{fig:noisy-25} and \ref{fig:noisy-50}, respectively.
Figure~\ref{fig:noisy-high} presents the  estimation error of \AlgDC and \AlgRC as we increase the number of  queries for $d=1000$ and  noise levels $\rho = 0.01, 0.1, 0.2$. 
 }
 \vspace{-.2cm}
\end{figure*}
In this section, we extensively evaluate the performance of
 \AlgDC against the following baselines:
 
 \AlgRS:  Queries are generated by sampling  uniformly at random from the unit sphere $S^{d-1}$.
 
 \AlgUS:  Queries are sampled uniformly at random from the orthogonal complement of $w$, where $w$ is the vector learned by linear SVM. 
 
 \AlgQB: The bag size is set to $20$ and  $1000$ queries are generated at each iteration. The query   with the largest disagreement is picked \cite{mamitsuka1998query}.
 
 \AlgSP: The version space is approximated by the largest ellipsoid consistent with all previous query-label pairs. Then, at each iteration a query is selected to approximately halve the ellipsoid \cite{AAAI}. 
 
 \AlgRC:  In the noisy setting, one easy way to apply \AlgDC is   to query each   point  $R$ times and use the majority rule to determine its label; i.e., the combination of repetitive querying (Section~\ref{sec:DC2}) and the \AlgDC framework (Section~\ref{sec:DC}).
%
%

Our metrics to compare different algorithms are: a) estimation error, b) query complexity, and c) execution time. In particular, as we increase the number of queries we measure the average estimation errors  and execution times for all the baselines (with $90\%$ confidence intervals). By nature, in active learning via query synthesis, all  data points and queries are generated synthetically. For all the baselines, we used the fastest available implementations in MATLAB. 

\textbf{Noiseless setting:} 
 Figures~\ref{fig:nf-25d} and \ref{fig:nf-50d}
   (with dimension $d= 25$ and $ 50 $, respectively)  show that in terms of estimation error, \AlgDC outperforms all other baselines, and significantly outperforms \AlgRS, \AlgUS and  \AlgQB. Note that the estimation errors are plotted in log-scales. In terms of execution times, we see in Fig.~\ref{fig:nf-time} that \AlgDC runs  three orders of magnitude
faster than other baselines. Training an SVM at each iteration for  \AlgRS, \AlgUS and \AlgQB comes with a huge computational cost. Similarly, \AlgSP requires solving a convex optimization problem at each iteration;  thus its performance drastically deteriorates as the dimension increases, which makes it infeasible  for many practical problems. 

%
%
%
%


\textbf{Noisy setting:} We set the  noise level to $\ep=0.1$ and compare the performance of \AlgDC against \AlgRS, \AlgUS, \AlgQB, and \AlgRC.
 As mentioned in \cite{AAAI}, and we have also observed in our experiments, \AlgSP does not work even for small amounts of noise as it incorrectly shrinks the version space and misses the true linear separator; therefore it is excluded here. We see again in Figures \ref{fig:noisy-25} and \ref{fig:noisy-50}   (for $d=25 $ and $50$) that \AlgDC significantly outperforms all other methods in terms of estimation error. More precisely, using the same number of queries, the estimation error of \AlgDC is around two orders of magnitude smaller than other baselines.
We can also observe from these two figures 
    that \AlgDC still runs around 100 times faster than \AlgRS, \AlgUS, and  \AlgQB. Clearly, \AlgDC has a higher computational cost than \AlgRC, as  \AlgDC performs a Bayesian update after each query. Finally, as we increase the dimension to $d=1000$,    \AlgRS, \AlgUS, and \AlgQB become significantly slower. Hence, in   Figure~\ref{fig:noisy-high} we only show how the estimation error (for noise levels $\rho=0.01, 0.1, 0.2$) decreases for \AlgDC and \AlgRC with more queries. It can be observed from Figure~\ref{fig:noisy-high} that consuming the same number of queries, \AlgDC can achieve an estimation error from one order (when the noise intensity is very small) to three orders of magnitude (when the noise intensity is $ 0.2 $) smaller than that of \AlgRC.
\section{Related Work}
\label{sec:related}

The sample complexity of learning a hypothesis was traditionally studied in the context of probably approximately correct (PAC) learning \cite{valiant1984theory}. In PAC learning theory, one assumes that a  set of hypotheses $\mathcal{H}$ along with  a set of unlabeled data points $\mathcal{X}$ are given, where each data point $x\in \mathcal{X}$ is drawn i.i.d.\ from some  distribution $D$.
Classical PAC bounds then yield the sample complexity  (i.e., the number of required i.i.d.\ examples)  from $D$ to output a hypothesis $h\in \mathcal{H}$ that will have  \textit{estimation error} at most $\epsilon$ with probability at least $1-\delta$, for some fixed $\epsilon,\delta>0$. Here, the estimation error is defined as $\epsilon = \Pr_{x\sim D}[h(x)\neq \h(x)]
$, where $\h$ is the unknown true hypothesis.
In the \textit{realizable} case of learning a  halfspace, i.e., when $\h\in\R^d$  perfectly separates the data points into  positive and negative labels, it is known that with $\tilde{O}(d/\epsilon)$\footnote{ {We use the $\tilde{O}$ notation to ignore terms that are logarithmic or dependent on $\delta$.}} i.i.d.\ samples one can find a linear separator with an estimation error $\epsilon$. The main advantage of using active learning methods, i.e., sequentially querying data points, is to reduce the sample complexity exponential fast, ideally to $\tilde{O}(d\log(1/\epsilon))$. In fact, a simple counting argument based on sphere packing shows that any algorithm needs $\Omega(d\log(1/\epsilon))$ examples to achieve an estimation error of $\epsilon$ \cite{dasgupta2009analysis}. 

For $d=2$ and when the  distribution is uniform over the unit sphere $S^1$ it is very easy to see that the halving or bisection leads to $\tilde{O}(\log(1/\epsilon))$. By using the same halving method, one can in principle extend the result to any dimension $d$. To do so, we need to carefully construct the version space (i.e., the set of hypotheses consistent with the queries and outcomes) at each iteration and then find a query that halves the volume (in the uniform case) or the density (in the general case if the distribution is known) \cite{dasgupta2004analysis}. Finding such a query in high dimension is  very challenging. 

One very successful approach that does not suffer from the aforementioned computational challenge is pool-based active learning  \cite{settles2010active}, where instead of ideally halving the space, effective approaximations are performed. Notable algorithms are \textit{uncertainty sampling} \cite{lewis1994sequential} and   query-by-committee (QBC) \cite{freund1997selective}.
In fact, our problem is closely related to learning  homogeneous linear separators under the uniform distribution in the pool-based setting. This problem is very well understood and there exist efficient pool-based  algorithms \cite{balcan2007margin,dasgupta2005perceptron,dasgupta2008hierarchical}. In particular, Dasgupta et al.~\cite{dasgupta2009analysis} presented an efficient perceptron-based algorithm  that achieve a near-optimal query complexity. Similar results can be obtained under log-concave distributions \cite{balcan2013active}.  Most of the pool-based methods require to have access to $\tilde{O}(1/\epsilon)$ number of unlabeled samples in each iteration or otherwise they perform very poorly \cite{balcan2007margin,dasgupta2009analysis}. This means that in order to have exponential guarantee in terms of sample complexity, we need to grow the pool size exponentially fast (note that there is no need to store all of these points). Moreover, with a few exceptions \cite{awasthi2014power,balcan2006agnostic} pool-based  learning of linear separators  in the noisy setting has been  much less studied and the dependency of sample complexity on noise is not very well understood. 
An attractive  alternative to the pool-based framework is  query synthesis where we have access to  membership queries~\cite{angluin1988queries}): a learner can  request  for any unlabeled data instance from the input space,  including queries that the learner synthesizes  from scratch. This way  the pool size limitation is entirely eliminated.  In many  recent applications, ranging from automated science \cite{king2009automation}, to robotics \cite{cohn1996active}, and  to adversarial reverse engineering \cite{lowd2005adversarial}, query synthesis is the appropriate  model. 
For instance, in security-sensitive applications (e.g., spam filters and intrusion detection systems) that routinely use machine learning tools,  a growing concern is the ability of adversarial attacks to identify the blind spots of the learning algorithms.  Concretely,  classifiers are commonly deployed  to detect miscreant activities. However, they are  attacked by adversaries who generate exploratory  queries   to elicit information that in return allows them to evade detection  \cite{nelson2012query}. In this work, we show how an adversary can use active learning methods by making synthetically \textit{de novo} queries and thus identify the linear separator used for classification.  We should emphasize that in active learning via \textit{synthesized} queries  the learning algorithm can  query the label of \emph{any points}  in order to  explore the hypothesis space.  In the noiseless setting (if we ignore the dependency of the pool size  on  $\tilde{O}(\log(1/\epsilon))$), one can potentially use the pool-based algorithms (under the uniform distribution). Our main contribution in this paper is to develop a noise resilient active learning algorithm that has access to \textit{noisy} membership queries. To the best of our knowledge, we are the first to show a near optimal algorithm that outperforms in theory and practice the naive repetition mechanism and the recent spectral heuristic methods \cite{AAAI}. 
	\clearpage

	\bibliographystyle{aaai}
	\bibliography{reference-list}
	
	\clearpage
\appendix
\section*{Appendix}
\renewcommand{\thesubsection}{\Alph{subsection}}


\subsection{Proof of Theorem~\ref{thm:noisy}}\label{sub:noisy}
%
%

Let $\{\zeta_n, n\geq 1\}$ be a sequence of independent and identically distributed (iid) Bernoulli($\rho$) 
random variables. Denote by $(\mathcal{F}, \Omega, {\rm{Pr}})$ the 
probability space generated by 
this sequence. At the $m$-th round of \AlgDCt, if $\zeta_m = 1$ (which takes place with independent probability $\rho$) 
then we observe a flipped version of $\sgn \langle x_m, h^* \rangle$. Also, if 
$\zeta_m = 0$ we observe the correct version of $\sgn \langle x_m, h^* \rangle$.

Consider a query of the form $\sgn\langle x,\h\rangle$. This query
divides the unit circle into two parts (half-circles) 
depending on the sign of $\langle x,\h\rangle$ (see Figure~\ref{figproof1}). The two parts
are:
(i) \textit{Preferred part}: all $h$ such that $\sgn\langle x, h\rangle=\sgn\langle x, h^\perp \rangle$, and
(ii) \textit{Unpreferred part}: all $h$ such that $\sgn\langle x, h\rangle=-\sgn\langle x,h^\perp \rangle$.
The two parts can be separated by a line $\ell_{x}$ that passes through the origin. We refer
to Figure~\ref{figproof1} for a schematic explanation.
  \begin{figure}[ht!]
 \begin{center} 
  \includegraphics[width=6cm]{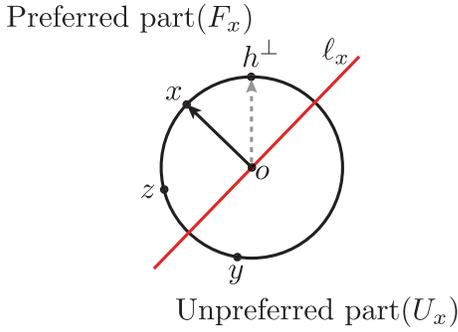}
 \end{center}
 \caption{ For any point $z$ above the line $\ell_x$ we have $\langle z, h^\perp \rangle = \langle x, h^\perp \rangle$. Once we perform the query $\langle x, h^\perp \rangle$, it is more likely that the (noisy) response is indeed the true value $\langle x, h^\perp \rangle$. Therefore, the region above the line $\ell_x$ is in general preferred by the query. In the figure, the sector $(y,z)$ is cut by the line $\ell_x$ and the sector $(z,x)$ is not. Also, $(z,x)$ lies in the preferred part of the query $\langle x, h^\perp \rangle$.}
 \label{figproof1}
 \end{figure}

In this setting, we say that the query $\sgn\langle x,\h\rangle$
\emph{prefers} a point $z$ if $z$ belongs to the preferred part of the query.
Otherwise, we say that the query \emph{does not prefer} $z.$ Also,
we frequently use the line $\ell_{x}$ rather than the query $\sgn\langle x,\h\rangle$
when it causes no ambiguity. Finally,  for a region $A$ on the unit circle  
say that the query $\sgn\langle x,\h\rangle$ \emph{cuts} the region $A$ if and only if the
line $\ell_{x}$ passes through region $A$. Otherwise, we say that the
query \emph{does not cut }$A$. If $\ell_{x}$ does not cut $A,$ then 
$\ell_{x}$ prefers $A$ if $A$ is in the preferred
part and does not prefer $A$ otherwise (see Figure~\ref{figproof1}). 
Finally, for two points $x,y$ we define the distance $d(x,y)$ to be the length of the (smaller) sector 
between them (see Figure~\ref{figproof1}). Clearly, we have $d(x,y) \geq \Vert x  - y\Vert^2$. 

At round $m$ of \AlgDCt a vector $x_m$ is chosen and the (noisy) outcome of $\sgn\langle x_m,\h\rangle$ 
is observed.  As explained in Section~\ref{sec:DC2}, $x_m$ is chosen in a way that the preferred and unpreferred parts have equal 
measures under $p_{m-1}$, i.e., $p_{m-1}(F_{x_m}) = p_{m-1}(U_{x_m}) = \frac 12$.  
Let us see what happens to $p_m$ (the posterior belief about $h^\perp$ at round $m$) after we conduct the query $\sgn\langle x_m,\h\rangle$.
As the result of the query is noisy, we have two different update rules depending on each of the following cases: 
(i)  $\zeta_m =0$, i.e., we observe the correct value $\sgn\langle x_m,\h\rangle$.  
In this case, the measure $p_{m}$
is updated as follows
\[
p_{m+1}(h)=\begin{cases}
2(1-\rho)p_{m}(h) & \text{if \ensuremath{h\in F_{x_m}},}\\
(2\rho) p_{m}(h) & \text{if \ensuremath{h\in U_{x_m}}}.
\end{cases}
\]
(ii) $\zeta_m =1$, i.e., we observe the flipped value $-\sgn\langle x_m,\h\rangle$.  
In this case, the
measure $p_{m}$ is updated as follows
\[
p_{m+1}(h)=\begin{cases}
(2\rho) p_{m}(h) & \text{if \ensuremath{h\in F_{x_m}},}\\
2(1-\rho)p_{m}(h) & \text{if \ensuremath{h\in U_{x_m}}}.
\end{cases}
\]

Consider the number $T_{\epsilon, \delta}$ given in \eqref{T}.  Our goal is to show that 
\begin{equation}
\p\left[\exists y \in S^1 : d(y,h^\perp)>\r\text{ and }p_{\t}(y)\geq p_{\t}(h^\perp)\right]<\delta.\label{eq:relation}
\end{equation}
Clearly, the result of the theorem follows from \eqref{eq:relation}. 
For better illustration, we assume w.l.o.g that $h^\perp = (0,1)$. Consider a point $y$ on the right-hand side of the unit circle such that $d(y, h^\perp) > \frac{\epsilon}{2}$.  
Also, Consider points $z_0,z_K$ such that $d(z_0, h^\perp) = \epsilon /4$ and $d(h^\perp, z_K) = \epsilon/2$. We now divide the sector starting with $z_0$ and ending with $z_K$ into $K:= \t+1$ pints.    
That is, for $i = 1, 2, \cdots, K$ we denote by $z_i$ the point that $d(h^\perp, z_i) = \frac{\epsilon}{4} + i \frac{\epsilon}{4(\t+1)} $ (see Figure~\ref{fig6}).
  \begin{figure}[ht!]
 \begin{center} 
  \includegraphics[width=5cm]{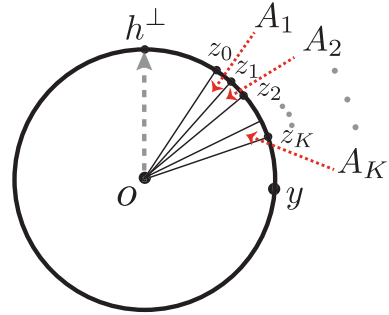}
 \end{center}
 \caption{Different regions for the  proof of Theorem~\ref{thm:noisy}.}\label{fig6}
 \end{figure}
 Also, for $i \geq 1$, we  let the sector starting with $z_{i-1}$ and ending with $z_{i}$ be denoted by $A_i$.   
Note that in the very beginning of the algorithm when
we have uniform measure on the unit circle, each of the regions $A_{i}$
has $p_{0}(A_{i})=\frac{\epsilon}{8\pi\cdot(\t+1)}$ (as $|A_{i}|=\frac{\epsilon}{4(\t+1)}$). 

\AlgDCt has in total $\t$ rounds and in each round $m$ it conducts a query with an associated line $\ell_{x_{m}}$. 
We let $M := \lceil\frac{2\log\frac{2}{\delta}}{\log 4(\ep (1-\ep))}\rceil$ and consider the following events:
\begin{itemize}
\item $E_{1}$: There is at least $M$ lines which separate
$z_{K}$ from $h^\perp$ or equivalently, there is at least $M$ lines that
cut the region $(h^\perp,z_{K}).$ 
\item $E_{2,j}$ ($1\leq j\leq K$): The region $A_{j}$ is \emph{not}
cut by any of the lines $\ell_{1},\ell_{2},\ldots,\ell_{\t}.$
\item $E_{3}$: $\exists y$ such that $d(y, h^\perp)> \frac{\rho}{2}$ and $p_{\t}(y)\geq p_{\t}(h^\perp)$.
\end{itemize}
It is easy to see that $\p\left[\bigcup_{j=1}^{K}E_{2,j}\right]=1$
as we have $\t$ queries and hence by the pigeon-hole principle there is always a region $A_j$ that is not cut by any of the lines.
We can write:
\begin{eqnarray} \nonumber
 &  & \p\left[E_{3}\right]\label{eq:eq2}\\ \nonumber
 & = & \p\left[E_{3}\cap E_{1}\right]+\p\left[E_{3}\cap E_{1}^{c}\right]\label{eq:secondline}\\
 & \leq & \p\left[E_{3}\mid E_{1}\right]+\sum_{j=1}^{\t+1}\p\left[E_{3}\cap E_{1}^{c}\cap E_{2,j}\right].\label{eq:thirdline}
\end{eqnarray}
Now using  Lemma \ref{lem:1} (stated below), we have 
\begin{equation}
\p\left[E_{3}\mid E_{1}\right]\leq\p\left[E_{3}\mid E_{1}\right]\leq \left( 4 \ep(1-\ep) \right)^\frac{M}{2} \leq \frac{\delta}{2} .\label{eq:eq3}
\end{equation}
Let us now bound $\p\left[E_{3}\cap E_{1}^{c}\cap E_{2,j}\right].$
We have
\[
\p\left[E_{3}\cap E_{1}^{c}\cap E_{2,j}\right]\leq\p\left[E_{2,j}\cap E_{1}^{c}\right],
\]
and using the fact that $|E_{2,j}|=\frac{\r}{4(\t+1)}$ we obtain
from Lemma \ref{lem:2} that
\begin{equation*}
\p\left[E_{2,j}\cap E_{1}^{c}\right]\leq (M-1) 
(\theta_1 + \theta_2),
\end{equation*}
and thus
\begin{equation} \label{final-bound}
\sum_{j=1}^{\t+1}\p\left[E_{2,j}\cap E_{1}^{c}\right] 
\leq(\t+1) M (\theta_1 + \theta_2),
\end{equation}
where $\theta_1$ and $\theta_2$ are given in Lemma~\ref{lem:2} with $m \leftarrow T_{\rho, \delta}$ and $k \leftarrow M$. 
Now, we show that the above expression is upper bounded by $\delta/2,$
and hence by using relations 
\eqref{eq:thirdline} and (\ref{eq:eq3}), we get the proof of the
main theorem.

The value of $T_0$  is chosen in such a way that we have 
\begin{equation} \label{e1}
\frac{2 \log(\t +1)}{\t - M}  \leq \frac{\log (2(1-\rho))}{4}.
\end{equation} 
$T_1$ ensures that
\begin{equation}
\frac{2}{\t - M} \log \frac{8\pi}{\epsilon}  \leq \frac{\log (2(1-\rho))}{4}.
\end{equation} \label{e2}
$T_2$ and ensures that
\begin{equation}  \label{e3}
\frac{2M}{\t - M} \log(2\rho) \leq \frac{\log(1-2\rho)}{2}.
\end{equation}
Finally, $T_3$ ensures that
\begin{equation} \label{e4}
(\t + 1) M \exp\left\{ -  \ep \frac{T-M}{6} \left(\frac{\log(2(1-\ep)) }{2\ep \log \frac{1-\rho}{\rho} }\right)^2 \right\} \leq \frac{\delta}{4}. 
\end{equation}
Now, by plugging in \eqref{e1}-\eqref{e4} into the values of $\theta_1$ and $\theta_2$ in \eqref{final-bound} we conclude that the right side of \eqref{final-bound} is bounded by $\frac{\delta}{2}$.

%
%
%
\begin{lem}
\label{lem:1} 
Let $x_1, x_2, \cdots, x_m$ be the vectors chosen by \AlgDCt up to round $m$
with $F_{x_i}$ and $U_{x_i}$ being their associated  preferred and unpreferred parts 
 (i.e. $p_{i-1}(F_{x_i}) = p_{i-1}(U_{x_i}) = 1 / 2$). 
Consider two points $h_1, h_2$ such that $h_1 \in \cap_{i=1}^m F_{x_i}$ and $h_2 \in \cap_{i=1}^m U_{x_i}$. We have for $\beta > 0$ that 
%
\[
\p\left[ p_m(x) < p_m(y) \right]  \leq \left(4 \ep (1-\ep) \right)^m.
\]
\end{lem}
\begin{proof}
For $i \in [m]$, define the random variable $Z_i$ as $Z_i \triangleq \log \frac{p_i(x)}{p_i(y)} $. Using the update rules of $p_i$ that we explained above, it is easy to see that for $i \geq 1$: 
$Z_i = Z_{i-1} + (1 - 2 \zeta_i)\log \frac{1-\ep}{\ep} $. Also, as $p_0$ is uniform over $S^1$ we have  $Z_0 = 0$. We thus 
have $Z_m = \sum_{i=1}^m(1-2 \zeta_i)\log \frac{1-\ep}{\ep}$.
Hence,
\begin{eqnarray*}
& & \text{Pr} \left[Z_m \leq 0\right] \\
&=& \text{Pr} \left[\log \frac{1-\ep}{\ep} \sum_{i=1}^m (1-2 \zeta_i) \leq 0 \right]\\
&= &\text{Pr} \left[\sum_{i=1}^m \zeta_i \geq  \frac 12 \right]\\
&\leq & (4 \ep (1-\ep))^\frac{m}{2},
\end{eqnarray*}
where the last step follows directly from the so called Chernoff bound. 
\end{proof}
We note that the vector $h^\perp$ is always a member of the preferred part of any test. As a result, at any round of \AlgDCt we have that $h^\perp \in \cap_{i=1}^m F_{x_i}$.   

\begin{lem}
\label{lem:2}Consider a region $A$ on the unit circle which does
not contain $h^\perp$. Assume we are at round $m$ of \AlgDCt where
a sequence of queries with associated lines $\ell_{x_1},\ell_{x_2},\ldots,\ell_{x_m}$ have been conducted. 
We define events $E_{1}$ and
$E_{2}$ as
\begin{itemize}
\item $E_{1}\triangleq\text{None of the lines \ensuremath{\ell_{x_i}} cuts \ensuremath{A}};$
\item $E_{2}\triangleq\text{At most \ensuremath{k} of the lines do not prefer \ensuremath{A}}$,
\end{itemize}
where $k$ is an an integer. We have
\begin{equation*}
\p\left[E_{1}\cap E_{2}\right]\leq k (\theta_1+\theta_2),
\end{equation*}
where
\begin{equation*} 
\theta_1 = \exp \left\{- \ep \frac{m-k}{6} \left(\frac{\log(2(1-\ep)) - \frac{2}{m-k} \log \frac{2\pi}{|A|})}{\ep \log(\frac{1-\ep}{\ep})} \right)^2 \right\},
\end{equation*}
and
\begin{equation*} 
\theta_2  = \exp \left\{- \ep \frac{m-k}{6} \left(\frac{\log(2(1-\ep)) + \frac{2k}{m-k} \log (2\ep)}{\ep \log(\frac{1-\ep}{\ep})} \right)^2 \right\}.
\end{equation*}
\end{lem}
\begin{proof}
We have
\begin{equation}
\p\left[E_{1}\cap E_{2}\right]\leq\p\left[E_{2}\mid E_{1}\right]\leq\sum_{j=1}^{k}\p\left[E_{2,j}\mid E_{1}\right],\label{eq:1}
\end{equation}
where we define 
\[
E_{2,j}\triangleq\text{Exactly \ensuremath{j} lines do not prefer \ensuremath{A}.}
\]
We will now calculate $\p\left[E_{2,j}\mid E_{1}\right].$ In the
beginning, $p_{0}$ puts a uniform measure on $A$ and hence $p_{0}(A)=\frac{|A|}{2\pi}.$
Let us first investigate the dynamics of $p_{i-1}(A)$ when we conduct
the $i$-th query and condition on event $E_1$ (i.e. given that none of the lines cut $A$).
In this setting, we define the random variables $Z_i = \log p_{i} (A)$.
At time $i$, assuming that the line $\ell_{x_i}$ does not cut $A$, $Z_i$ has different update rules depending on 
the two cases whether  the line $\ell_{x_i}$  prefers $A$ or does not prefer $A$.
(i) first case: if the line $\ell_{x_i}$ prefers
$A$, then we know that either with
probability $1-\rho$ (if $\zeta_i = 0$) we have $p_{i}(A)=2(1-\ep)p_{i-1}(A)$ and
with probability $\ep$ (if $\zeta_i = 1$) we have $p_{i}(A)=(2\rho) p_{i-1}(A)$. 
Thus, we can write $Z_i = Z_{i-1} + F_i$, where $F_i \triangleq \zeta_i \log (2 \rho) + (1-\zeta_i) \log(2(1-\rho))$. 
(ii) second case: if $\ell_{x_i}$ does not prefer $A$, then using a  similar argument we obtain
$Z_{i} = Z_{i-1} + U_i$, where $U_i \triangleq \zeta_i \log (2(1- \rho)) + (1-\zeta_i) \log(2 \rho )$. 
Now, in order to find an upper bound on $\p\left[E_{2,j}\mid E_{1}\right]$,
we assume without loss of generality that in the first $m-j$ rounds we the lines are as in the first case and 
in the last $j$ rounds the lines are as in the second case (note that any other given order of the lines is statistically 
equivalent to this simple order that we consider). 
\begin{eqnarray*}
Z_m & = & Z_0 + \sum_{i=1}^{m-j}F_{i}+\sum_{i=m-j+1}^{m}U_{i}\\
 & = & \log_{2}\frac{|A|}{2\pi} + \sum_{i=1}^{m-j}F_{i}+\sum_{i=m-j+1}^{m}U_{i}.
\end{eqnarray*}
Now, noting that $p_{m}(A)\leq1$ and hence $\log p_{m}(A)\leq0,$
we obtain 
\begin{eqnarray*}
 &  & \p\left[E_{2,j}\mid E_{1}\right]\\
 & \leq & \p\left[\log_{2}p_{0}(A)+\sum_{i=1}^{m-j}F_{i}+\sum_{i=m-j+1}^{m}U_{i}\leq0\right]\\
 & = & \p\left[\sum_{i=1}^{m-j}F_{i}+\sum_{i=m-j+1}^{m}U_{i}  \leq\log_{2}\frac{2\pi}{|A|}\right]\\
\end{eqnarray*}
Let us now define
\begin{equation*}
\alpha_1 = \text{Pr}\left [ \sum_{i=1}^{\frac{m-j}{2}} F_i \leq \log \frac{2 \pi}{A}\right]
\end{equation*}
and 
\begin{equation*}
\alpha_2 = \text{Pr}\left [ \sum_{i=1+\frac{m-j}{2}}^{m-j} F_i +  \sum_{i=m-j+1}^{m} U_i   \leq 0\right]
\end{equation*}
Using the union bound, we have
\begin{equation}
\p\left[E_{2,j}\mid E_{1}\right] \leq \alpha_1 + \alpha_2. 
\end{equation}
Now, to bound $\alpha_1$ we obtain after some simplifications that
\begin{equation*}
\alpha_1 = \p\left[ \sum_{i=1}^{\frac{m-j}{2}} \zeta_i \geq 
\ep \times \frac{m-j}{2} \times \frac{ \log(2(1-\ep)) - \frac{2}{m-j}\log \frac{2\pi}{|A|}}{\ep \log \frac{1-\ep}{\ep}} \right],
\end{equation*}
and by using the Chernoff bound we get
\begin{equation} \label{alpha_11}
\alpha_1 \leq \exp \left\{- \ep \frac{m-j}{6} \left(\frac{\log(2(1-\ep)) - \frac{2}{m-j} \log \frac{2\pi}{|A|})}{\ep \log(\frac{1-\ep}{\ep})} \right)^2 \right\}.
\end{equation}
To bound $\alpha_2$ we can similarly write after some simple steps that
\begin{eqnarray*}
\alpha_2 \leq \p\left[   \sum_{i= 1+\frac{m-j}{2}}^{m-j} \zeta_i   \geq   \ep \times \frac{m-j}{2} \frac{ \log (2(1-\ep)) +\frac{2 j}{m-j} \log (2\ep) }{\ep \log \frac{1-\ep}{\ep}} \right],
\end{eqnarray*}
and using the Chernoff bound we get
\begin{equation} \label{alpha_21}
\alpha_2  \leq \exp \left\{- \ep \frac{m-j}{6} \left(\frac{\log(2(1-\ep)) + \frac{2j}{m-j} \log (2\ep)}{\ep \log(\frac{1-\ep}{\ep})} \right)^2 \right\}.
\end{equation}
We further note that both of the upper bounds on $\alpha_1$ and $\alpha_2$ decrease when we increase $j$. Hence, the proof of the theorem follows by letting $j = k$ in \eqref{alpha_11} and \eqref{alpha_21}, and also plugging these bounds into \eqref{eq:1}. 

%
\end{proof}

\subsection{\AlgDCt in the Noiseless Case}\label{sub:noiseless}

\AlgDCt (outlined in Algorithm~\ref{alg:DC2}) in the noiseless case reduces to the binary search.
In this section, we explain \AlgDCt  in the noiseless case (the binary search) with the help of a running example 
given in Figure~\ref{fig:runn-ex}. As we will see, after each round of \AlgDCt the possible region that $h^\perp$ can belong to will be ``halved''. 
\begin{figure}[t!]
	\begin{center} 
		\includegraphics[width=8cm]{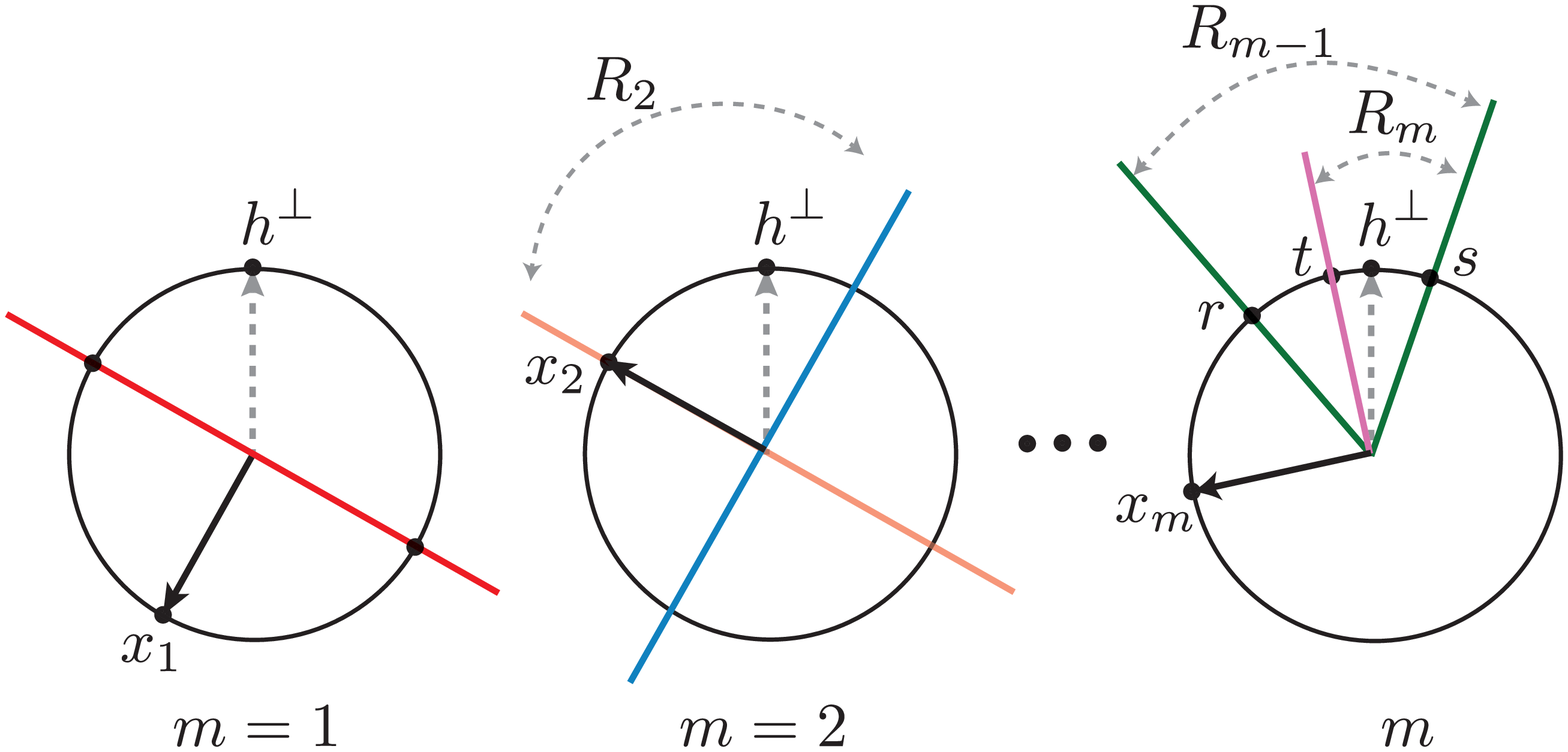}
	\end{center}
	\caption{\footnotesize An example to illustrate \AlgDCt in the noiseless setting.  
		In the first round,  $x_1$ is arbitrarily chosen on $S^1$. For the choice in the figure, we have  $\sgn \left\langle x_1, h^{*}\right\rangle = \sgn \left\langle x_1, h^{\perp}\right\rangle =-1$.  For any point $h$ above the red line we have that $\sgn \left\langle x, h\right\rangle = -1$ and for the points outside this half-circle the result is $+1$. Therefore, the distribution (pdf of) $p_1$ is uniform on the region above the red line and is zero below it. For round $m=2$ it is easy to see that the direction of $x_2$  should  be along the red line. For $x_2$ chosen as in the figure, we have  $\sgn \left\langle x_2, h^{*}\right\rangle = +1$ and hence at the end of the second round \AlgDCt concludes that the vector $h^\perp$ could uniformly be any point inside $R_2$. In a generic round $m$, any vector orthogonal to the mid-point of sector $R_{m-1}$ can be considered as a candidate for $x_m$.   For the choice in the figure, we have $\sgn \left\langle x_m, h^{\perp}\right\rangle = -1$. Thus, at the end of round $m$, \AlgDCt concludes that  $h^\perp$ can uniformly be any point inside $R_m$.   \label{fig:runn-ex}} 
\end{figure}

We first note that as the initial distribution $p_0$ is assumed to be the uniform distribution on $S^1$, the vector $x_1$ (see step 2-(a) of Algorithm~\ref{alg:DC2}) can indeed be any point on the unit circle $S^1$. Thus, \AlgDCt chooses $x_1$ arbitrarily on $S^1$. 
By \eqref{eq-proj}, using the query $\sgn \left\langle x_1, h^{*}\right\rangle$ will also give us the value of    
$\sgn \left\langle x_1, h^{\perp}\right\rangle$. Depending on this value, it is easy to verify that only half of $S^1$ can possibly contain $h^\perp$ (see Figure~\ref{fig:runn-ex}). Let us denote this region by $R_1$. Hence, the probability distribution $p_1(h)$ (which is our current belief about $h^\perp$) is updated as follows: for $h \notin R_1$ we have that $p_1(h) = 0$, and as all the points inside the half-circle $R_1$ are equiprobable, we have for $h \in R_1$ that $p_1(h) = 1 / \pi$. In other words, at time $m=0$ the vector $h^\perp$ could have been anywhere on the unit circle, but, after round $m=1$ it can only belong to the half-circle $R_1$. Thus, after the first round, \AlgDCt ``halves'' the admissible region of $h^\perp$. Continuing in this theme, it is not hard to verify that (see Figure~\ref{fig:runn-ex}) at round $m=2$ the value of $p_2(h)$ is non-zero and  uniform only on a region $R_2$ which is a quarter-circle. In an inductive manner, letting $R_{m-1}$ denote the admissible region (sector) at round $m-1$ (see Figure~\ref{fig:runn-ex}) and assuming that $p_{m-1}$ is only non-zero and uniform on the sector $R_{m-1}$, then $x_m$ at round $m$ is precisely the vector that is orthogonal to the midpoint of the sector $R_{m-1}$.  Therefore, after observing the value of $\sgn \left\langle x_m, h^{*}\right\rangle$, the admissible region $R_m$ is  the better half of $R_{m-1}$ that is compatible with the observation (i.e., it contains $h^\perp$). Also, $R_m$ is again a sector and $p_m$ will be uniform on $R_m$ and zero outside. It is also easy to see that the circular angle for the sector $R_m$ is $\frac{\pi}{2^m}$. The following statement is now immediate.
\begin{thm} \label{thm:noiseless}
	Consider \AlgDC in the  absence of noise ($\rho = 0$).  If we let $T_{\epsilon, \delta} = \lceil\log_{2}\frac{\pi}{\epsilon}\rceil$, then it outputs a vector that is within a distance $\epsilon$ of $h^\perp$. 
\end{thm}
A few comments are in order: The above guarantee for \AlgDCt holds with probability one and thus the parameter $\delta$ is irrelevant in the noiseless setting. Furthermore, during each round of \AlgDCt, the distribution $p_m$  can be represented by only two numbers (the starting and ending points of the sector $R_m$), and the vector $x_m$ can be computed  efficiently (it is the orthogonal vector to the midpoint of $R_m$). Therefore,  assuming one unit of complexity for performing the queries, \AlgDCt can be implemented with complexity $O(T_{\epsilon, \delta})$. Finally, by using Theorem~\ref{thm:DC}, we conclude that \AlgDC requires ${O}(d \log \frac{1}{\epsilon})$ queries with computational complexity  ${O}(d \log \frac{1}{\epsilon})$.

\subsection{Analysis of Repetitive Querying}\label{sub:rep}

Firstly, we would like to compute the probability that the majority vote gives us the correct outcome. Let $ I_i $ ($ 1\leq i\leq R $) is the indicator random variable of the event that the $ i $-th query gives us the right outcome. We know that $ \{I_i:1\leq i\leq R\} $ are i.i.d.\ Bernoulli random variables with success probability $ 1-\rho $. Let $ S_R = \sum_{i=1}^{R} I_i $. By Hoeffding's inequality, we have\begin{eqnarray*}
& & \Pr[S_R \leq R/2]\\
 & = & \Pr[S_R - E[S_R] \leq R/2-E[S_R]]\\
 &= & \Pr[S_R - E[S_R] \leq - R(1/2-\rho)]\\
 & \leq & e^{-2(1/2-\rho)^2 R}.
\end{eqnarray*}
Suppose that the binary search queries $ n_0 $ (distinct) points in total throughout the entire procedure. By the union bound, the probability that all $ n_0 $ majority votes give us the right outcome is greater than or equal to $ 1-n_0 e^{-2(1/2-\rho)^2 R}$. In order to ensure that this probability is at least $ 1-\delta $, we need \[ R\geq \frac{\log(n_0/\delta)}{2(1/2-\rho)^2}. \] Therefore the total number of queries is at least \[ n_0 R \leq n_0 \frac{\log(n_0/\delta)}{2(1/2-\rho)^2}. \]
Recall that $ n_0=O(\log (1/\epsilon) ) $ (see Theorem~\ref{thm:noiseless}). Plugging this into the expression of $ n R_0 $, we obtain that the query complexity of repetitive querying is $ O(\log(1/\epsilon)(\log \log (1/\epsilon)+\log (1/\delta )) $.
\subsection{Proof of Theorem~\ref{thm:DC}}\label{sub:DC}

At each round, we replace two vectors in $E$, say $e_1$ and $e_2$, with the
output of $\AlgDCt (e_1, e_2, \epsilon, \delta)$;
then the cardinality of $E$ decreases by $1$. Therefore, each call to
\AlgDCt will result in the cardinality of $E$
decreasing by $1$. Initially, there are $d$ elements in $E$; when the
algorithm terminates, there is only one element (i.e., the final output of the
algorithm) in $E$. Thus throughout the entire process of the algorithm, the
cardinality of $E$ decreases by $d - 1$; therefore, there are $d - 1$ calls to
\AlgDCt. If the probability of success for
\AlgDCt is at least $1 - \delta$, then by the union
bound the probability of success of \AlgDC is at
least $1 - (d - 1) \delta$.

For the second part of the theorem, we prove a more general statement: Assume
that we run \AlgDC with an input being an
orthonormal set $\{ e_1, e_2, \ldots, e_t \}$ where $e_i, h^{\ast} \in
\mathbb{R}^d$ and $t \leqslant d$. We should note that the underlying space
remains the $d$-dimensional Euclidean space $\mathbb{R}^d$. We will prove
that \AlgDC outputs a vector that is close to the
normalized orthogonal projection of $h^{\ast}$ onto
$\span \{ e_1, e_2, \ldots, e_t \}$. More
precisely, we define
\[ h^{\perp} = \frac{\sum_{i = 1}^t \langle e_i, h^{\ast} \rangle
	e_i}{\sqrt{\sum_{i = 1}^t \langle e_i, h^{\ast} \rangle^2}} . \]
	Then \AlgDC runs in $d - 1$ rounds, calls
	\AlgDCt $d - 1$ times, and outputs with probability
	at least $1 - (d - 1) \delta$ a vector $\hat{h}$ for which $\| \hat{h} -
	h^{\perp} \| \leqslant 5 \epsilon d$. In exactly similarly way as discussed
	above, we can conclude that \AlgDC runs in $d - 1$
	time and uses \AlgDCt $d - 1$ times. Also, again by
	the union bound, with probability at least $1 - (d - 1) \delta$, all outputs
	of \AlgDCt are a close estimate (within distance
	$\epsilon$) of their corresponding objective. Thus, by assuming that all
	the calls of \AlgDCt have been successful (which
	happens with probability at least $1 - (d - 1) \delta$), we use an inductive
	argument to prove that $\| \hat{h} - h^{\perp} \| \leqslant 5 \epsilon (d-1)$.
	We use induction on $t$. For $t = 2$ the result is clear. We now prove the
	result when $t = \tau$ assuming that it holds for all $t < \tau$. Without loss
	of generality, assume that the algorithm calls
	$\AlgDCt (e_1, e_2, \epsilon, \delta)$ and the
	vectors $e_1$ and $e_2$ in $E$ willl be replaced by the output of
	$\AlgDCt (e_1, e_2, \epsilon, \delta)$ which we
	denote by $\hat{e}_1$. We can write
	\[ h^{\perp} = \sum_{i = 1}^{\tau} c_i e_i = \hat{c}_1 h_1^{\perp} + \sum_{i =
		3}^{\tau} c_i e_i, \]
		where $c_i = \langle h^{\perp}, e_i \rangle$, $\hat{c}_1 = \sqrt{c_1^2 +
			c_2^2}$ and $h_1^{\perp} = \frac{c_1 e_1 + c_2 e_2}{\sqrt{c_1^2 + c_2^2}}$.
			Note that $h_1^{\perp} = \frac{c_1 e_1 + c_2 e_2}{\sqrt{c_1^2 + c_2^2}}$ is
			precisely the normalized orthogonal projection of $h^{\ast}$ (and also
			$h^{\perp}$) onto $\span \{ e_1, e_2 \}$. Using the
			above notation, we have
			\[ \| \hat{e}_1 - h_1^{\perp} \| \leqslant \epsilon . \]
			Recall that after obtaining $\hat{e}_1$ (the output of
			$\AlgDCt (e_1, e_2, \epsilon, \delta)$), the
			algorithm will recursively call $\AlgDC (\hat{e}_1,
			e_3, e_4, \ldots, e_{\tau})$. Suppose that the output of this call is denoted
			by $\hat{h}^{\perp}$. By the assumption of the induction, the output
			$\hat{h}^{\perp}$ will be within the distance $5 \epsilon (\tau - 2)$ of
			the normalized orthogonal projection of $h^{\perp}$ onto
			$\span \{ \hat{e}_1, e_3, e_4, \ldots, e_{\tau}
			\}$, which we denote by $h'$. That is,
			\[ \| \hat{h}^{\perp} - h' \| \leqslant 5 \epsilon (\tau - 2) . \]
			We know that
			\[ h' = \frac{\langle h^{\perp}, \hat{e}_1 \rangle \hat{e}_1 + \sum_{i = 3}^\tau
				c_i e_i}{\sqrt{\langle h^{\perp}, \hat{e}_1 \rangle^2 + \sum_{i = 3}^\tau
					c_i^2}} . \]
					Now we will show that
					\[ \| h' - h^{\perp} \| \leqslant 5 \epsilon . \]
					If this is true, we will have
					\[ \| h^{\perp} - \hat{h}^{\perp} \| \leqslant \| \hat{h}^{\perp} - h' \| + \|
					h' - h^{\perp} \| \leqslant 5 \epsilon (\tau-1) \]
					and this completes the proof. Therefore, it suffices to show that $\| h' -
					h^{\perp} \| \leqslant 5 \epsilon$.
					
					We define $\beta = \sqrt{\langle h^{\perp}, \hat{e}_1 \rangle^2 + \sum_{i =
							3}^\tau c_i^2}$. Firstly, we have
							\begin{eqnarray}
						& &	\| h' - h^{\perp} \|\nonumber\\
						 & = & \left\| h^{\perp} - \frac{\langle h^{\perp},
								\hat{e}_1 \rangle \hat{e}_1 + \sum_{i = 3}^\tau c_i e_i}{\sqrt{\langle
									h^{\perp}, \hat{e}_1 \rangle^2 + \sum_{i = 3}^\tau c_i^2}} \right\| \nonumber\\
									& = & \left\| \frac{\beta h^{\perp} - \left( \langle h^{\perp}, \hat{e}_1
										\rangle \hat{e}_1 + \sum_{i = 3}^\tau c_i e_i \right)}{\beta} \right\|
										\nonumber\\
										& = & \left\| \frac{(\beta - 1) h^{\perp} + h^{\perp} - \left( \langle
											h^{\perp}, \hat{e}_1 \rangle \hat{e}_1 + \sum_{i = 3}^\tau c_i e_i
											\right)}{\beta} \right\| \nonumber\\
											& \leqslant & \left| \frac{1 - \beta}{\beta} \right| + \left\|
											\frac{h^{\perp} - \left( \langle h^{\perp}, \hat{e}_1 \rangle \hat{e}_1 +
												\sum_{i = 3}^\tau c_i e_i \right)}{\beta} \right\| . \label{eq:diff} 
												\end{eqnarray}
												Secondly, we have
												\begin{eqnarray*}
													| \beta^2 - 1 | & = & \left| \langle h^{\perp}, \hat{e}_1 \rangle^2 +
													\sum_{i = 3}^{\tau} c_i^2 - 1 \right|\\
													& = & \left| \langle h^{\perp}, \hat{e}_1 \rangle^2 + \sum_{i = 3}^{\tau}
													c_i^2 - \left( \hat{c}_1^2 + \sum_{i = 3}^{\tau} c_i^2 \right) \right|\\
													& = & | \langle h^{\perp}, \hat{e}_1 \rangle^2 - \hat{c}_1^2 |\\
													& = & | \langle h^{\perp}, \hat{e}_1 \rangle^2 - \langle h^{\perp},
													h_1^{\perp} \rangle^2 |\\
													& = & | \langle h^{\perp}, \hat{e}_1 \rangle - \langle h^{\perp},
													h_1^{\perp} \rangle | \cdot | \langle h^{\perp}, \hat{e}_1 \rangle + \langle
													h^{\perp}, h_1^{\perp} \rangle |\\
													& = & | \langle h^{\perp}, \hat{e}_1 - h_1^{\perp} \rangle | \cdot |
													\langle h^{\perp}, \hat{e}_1 \rangle + \langle h^{\perp}, h_1^{\perp}
													\rangle |\\
													& \leqslant & \| h^{\perp} \| \cdot \| \hat{e}_1 - h_1^{\perp} \| \cdot (\|
													h^{\perp} \| \cdot \| \hat{e}_1 \| + \| h^{\perp} \| \cdot \| h_1^{\perp}
													\|)\\
													& \leqslant & 2 \epsilon,
													\end{eqnarray*}
													where the last step follows from $\| h^{\perp} \| = \| \hat{e}_1 \| = \|
													h_1^{\perp} \| = 1$ and $\| \hat{e}_1 - h_1^{\perp} \| \leqslant \epsilon$.
													Since $\epsilon \leqslant 5 / 18 < 1 / 2$, we obtain
													\[ \beta \in \left[ \sqrt{1 - 2 \epsilon}, \sqrt{1 + 2 \epsilon} \right]
													\]
													and
													\begin{equation}
													\left| \frac{1 - \beta}{\beta} \right| \leqslant \max \left\{
													\frac{1}{\sqrt{1 - 2 \epsilon}} - 1, 1 - \frac{1}{\sqrt{1 + 2
															\epsilon}} \right\} \leqslant 2 \epsilon . \label{eq:2e}
															\end{equation}
															Thirdly, similarly we have
															\begin{eqnarray*}
																&  & \left\| h^{\perp} - \left( \langle h^{\perp}, \hat{e}_1 \rangle
																\hat{e}_1 + \sum_{i = 3}^\tau c_i e_i \right) \right\|\\
																& = & \left\| \sum_{i = 1}^\tau c_i e_i - \left( \langle h^{\perp}, \hat{e}_1
																\rangle \hat{e}_1 + \sum_{i = 3}^\tau c_i e_i \right) \right\|\\
																& = & \| c_1 e_1 + c_2 e_2 - \langle h^{\perp}, \hat{e}_1 \rangle \hat{e}_1
																\|\\
																& = & \| \langle h^{\perp}, h_1^{\perp} \rangle h_1^{\perp} - \langle
																h^{\perp}, \hat{e}_1 \rangle \hat{e}_1 \|\\
																& \leqslant & \| \langle h^{\perp}, h_1^{\perp} \rangle h^{\perp}_1 -
																\langle h^{\perp}, h_1^{\perp} \rangle \hat{e}_1 \| + \| \langle h^{\perp},
																h_1^{\perp} \rangle \hat{e}_1 - \langle h^{\perp}, \hat{e}_1 \rangle
																\hat{e}_1 \|\\
																& = & \| \langle h^{\perp}, h_1^{\perp} \rangle (h^{\perp}_1 - \hat{e}_1)
																\| + \| \langle h^{\perp}, h_1^{\perp} - \hat{e}_1 \rangle \hat{e}_1 \|\\
																& \leqslant & \| h^{\perp}_1 - \hat{e}_1 \| + \| h_1^{\perp} - \hat{e}_1
																\|\\
																& \leqslant & 2 \epsilon .
																\end{eqnarray*}
																Therefore
																\begin{eqnarray}
																&  & \left\| \frac{h^{\perp} - \left( \langle h^{\perp}, \hat{e}_1 \rangle
																	\hat{e}_1 + \sum_{i = 3}^\tau c_i e_i \right)}{\beta} \right\| \nonumber\\
																	& \leqslant & \frac{2 \epsilon}{\beta} \nonumber\\
																	& \leqslant & \frac{2 \epsilon}{\sqrt{1 - 2 \epsilon}} \nonumber\\
																	& \leqslant & 3 \epsilon, \label{eq:3e} 
																	\end{eqnarray}
																	where the last step follows from $\epsilon \leqslant 5 / 18$.
																	
																	Now, by plugging (\ref{eq:2e}) and (\ref{eq:3e}) into (\ref{eq:diff}) we get
																	\[ \| h' - h^{\perp} \| \leqslant 2 \epsilon + 3 \epsilon = 5
																	\epsilon . \]
																	Hence we have
																	\[ \| h^{\perp} - \hat{h}^{\perp} \| \leqslant \| \hat{h}^{\perp} - h' \| + \|
																	h' - h^{\perp} \| \leqslant 5 \epsilon (\tau - 2) + 5 \epsilon = 5
																	\epsilon (\tau-1) . \]

\end{document}